\documentclass[lettersize,journal]{IEEEtran}
\IEEEoverridecommandlockouts
\usepackage{cite}
\usepackage{amsmath,amssymb,amsfonts}
\usepackage{amsthm}
\usepackage{graphicx}
\usepackage{textcomp}
\usepackage{hyperref}       
\usepackage{url}            
\usepackage{booktabs}       
\usepackage{nicefrac}       
\usepackage{microtype}      
\usepackage{xcolor}         
\usepackage{etoolbox}
\usepackage{changepage}
\usepackage{titlesec}
\usepackage{algorithm}
\usepackage{algorithmic}
\usepackage{times}
\usepackage{float}
\usepackage{ragged2e}
\usepackage{subcaption}
\usepackage{verbatim}
\usepackage{diagbox}
\usepackage{threeparttable}
\usepackage{multirow}
\usepackage{multicol}
\usepackage{color, colortbl}
\usepackage{blindtext}
\usepackage{tabularx}
\usepackage[T1]{fontenc}

\theoremstyle{plain}
\newtheorem{theorem}{Theorem}
\newtheorem{proposition}[theorem]{Proposition}
\newtheorem{observation}[theorem]{Observation}

\theoremstyle{definition}

\theoremstyle{remark}

\def\BibTeX{{\rm B\kern-.05em{\sc i\kern-.025em b}\kern-.08em
    T\kern-.1667em\lower.7ex\hbox{E}\kern-.125emX}}
    
\begin{document}

\title{Robust Auto-associative Memory via \\ Convolutional Restricted Hopfield Networks}

\author{\IEEEauthorblockN{Ci Lin, Tet Yeap, Iluju Kiringa}\\
	\IEEEauthorblockA{\textit{Electrical Engineering and Computer Science} \\
		\textit{University of Ottawa}\\
		\{clin072, tyeap, iluju.kiringa\}@uottawa.ca} 
}

\markboth{Journal of \LaTeX\ Class Files,~Vol.~18, No.~9, September~2020}%
{How to Use the IEEEtran \LaTeX \ Templates}

\maketitle

\begin{abstract}
	Associative memory models play a fundamental role in pattern retrieval, but their performance often degrades under adversarial perturbations and severe input corruptions. Existing approaches, including Modern Hopfield Networks (MHNs), and Predictive Coding Networks (PCNs), exhibit limitations in balancing storage capacity, computational efficiency, and robustness. In this paper, we propose a Convolutional Restricted Hopfield Networks (CRHNs), which integrates convolutional feature extraction with attractor-based memory retrieval in a structured latent space. The proposed model leverages subspace representations and fixed-point dynamics, trained via a gradient-free Subspace Rotation Algorithm (SRA), to enhance both robustness and memory capacity. 
	
	Extensive experiments on Self-Taught Learning (STL) dataset demonstrate that CRHNs consistently achieve significantly lower reconstruction error compared to MHNs and PCNs across a wide range of adversarial attacks and input degradations. In many cases, CRHNs reduce reconstruction error by an order of magnitude and maintains stable retrieval performance under increasing perturbation strength. Statistical analysis further confirms that these improvements are significant ($p < 0.01$). These results highlight the effectiveness of attractor-based memory mechanisms and suggest that CRHNs provide a promising framework for building robust and scalable associative memory systems.
\end{abstract}


\section{Introduction}\label{sec:intro}

In recent years, several associative memory models have been proposed to address the capability limitations of classical Hopfield Neural Networks (HNNs) \cite{hopfield1982neural, hopfield1984neurons}. Dense Associative Memories (DAMs) introduces stronger nonlinearities in the energy function, aiming to improve the storage capacity and at the same to enhance the robustness \cite{krotov2016dense, krotov2020large}. Modern Hopfield Networks (MHNs) further extend this framework by establishing a connection to the attention mechanism in Transformers, enabling seamless integration into deep learning architectures \cite{ramsauer2020hopfield}. Meanwhile, Predictive Coding Networks (PCNs), inspired by theories of cortical inference, perform memory retrieval through hierarchical error minimization and demonstrate robustness to noise and partial observations \cite{salvatori2021associative, spratling2017review, tang2022associative}. However, although these alternative models achieve significant improvements over the original HNNs, they still struggle to maintain strong robustness against adversarial attacks and severe input corruption.
\begin{figure}[htbp]
	\centering
	\begin{subfigure}[t]{0.45\textwidth}
		\centering
		\includegraphics[width=\linewidth]{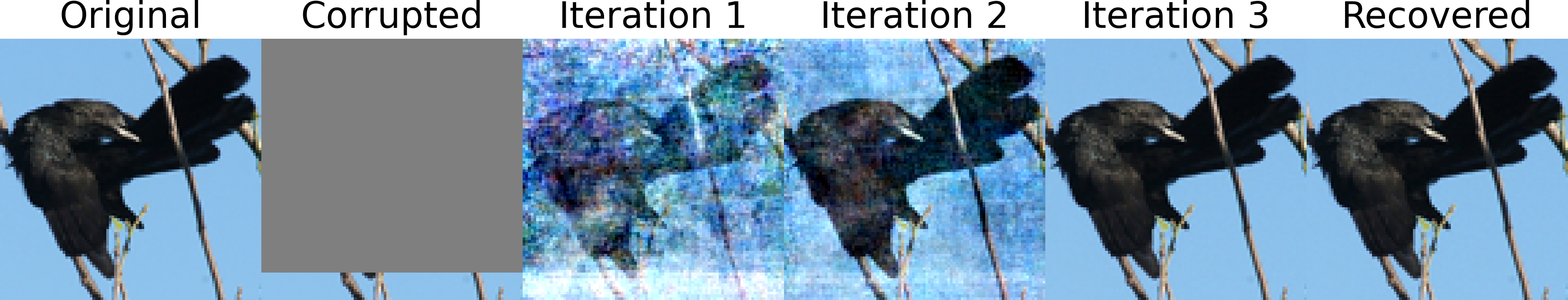}
	\end{subfigure}
	
	\vspace{0.5em}
	
	\begin{subfigure}[t]{0.45\textwidth}
		\centering
		\includegraphics[width=\linewidth]{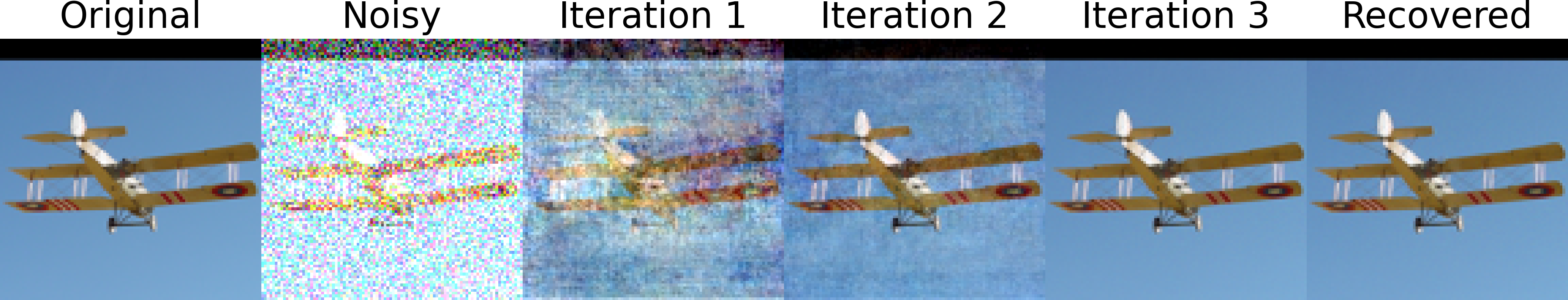}
	\end{subfigure}
	
	\caption{Iterative retrieval dynamics of Convolutional Restricted Hopfield Networks. Given corrupted and noisy inputs, the network progressively refines the patterns over iterations and converges to stored memories.}
	\label{fig:retrieval_dynamics}
\end{figure}
In biological systems, associative memory is a fundamental capability associated with the hippocampus, where recurrent neural dynamics support pattern completion from noisy sensory cues \cite{wahlheim2025memory, marr1971simple, rolls2013mechanisms}. In particular, the CA3 subregion, characterized by dense recurrent connectivity, has long been hypothesized to implement an auto-associative memory system with attractor dynamics \cite{hopfield1982neural, rolls1995model}. Combining both the mechanism of biological system and well-developed deep learning theory,   Restricted Hopfield Networks (RHNs) were introduced as an extension incorporating hidden units and dynamic training mechanisms to enhance storage capacity and retrieval performance \cite{yeap2021implementation, lin2025rhnrobust, lin2023basin, lin2024subspace}. It is observed that RHNs demonstrates strong resistance to adversarial attack and can easy retrieved the complete patterns from highly corrupted and noisy patterns \cite{goodfellow2014explaining, mkadry2017towards, lin2025rhnrobust}. While MHNs and PCNs exhibit certain degrees of robustness, their underlying mechanisms limit their effectiveness under strong or targeted attacks. This highlights the potential of RHNs as a foundation for designing adversarially resilient memory systems \cite{lin2025rhnrobust}.

Although original RHNs demonstrate highly resistant to the adversarial perturbation or naturally corruption. It is not straightforward to expand it into continuous samples and handle high-dimensional data. Therefore, this paper proposed Convolutional Restricted Hopfield Networks (CRHNs), which is closely resembles the pattern completion mechanism observed in the hippocampus, iteratively refines the corrupted or noisy pattern through recurrent interactions until convergence to a stable attractor corresponding to a stored memory. As illustrated in Fig.~\ref{fig:retrieval_dynamics}, this iterative inference process enables robust reconstruction even under severe degradation, providing a computational analogue to biological memory retrieval.

\subsection{Motivation and Contribution}

The study of adversarial robustness is critical for ensuring the reliability and security of deep learning systems in real-world applications \cite{ren2021adversarial}. Vulnerabilities to adversarial perturbations pose significant risks in domains such as healthcare, autonomous systems, and cybersecurity, where errors can have severe consequences \cite{newaz2020adversarial, apruzzese2019addressing, badjie2024adversarial}. Improving robustness not only enhances model reliability but also provides insights into the fundamental limitations of current architectures.

In this work, we propose a novel extension of RHNs, termed the CRHNs, designed to incorporate structured representations for high-dimensional data. The main contributions of this paper are summarized as follows:

\begin{itemize}
	\item A novel CRHNs framework is proposed to integrate convolutional feature extraction with RHNs dynamics. The proposed architecture enables the network to store and retrieve complex image patterns in a compact latent space while preserving spatial structures and semantic information.
	
	\item An iterative retrieval mechanism and subspace-rotation-based training strategy are developed to improve memory stability and robustness. By using orthogonality-preserving left and right subspace rotation updates on RHNs, the proposed method achieves improved convergence behavior and enhanced robustness against corruption, noise, brightness variation, and adversarial perturbations.
	
	\item Extensive experiments are conducted to compare the proposed CRHNs with MHNs and PCNs on STL dataset. Experimental results demonstrate that the proposed method achieves superior reconstruction quality and robustness under severe image degradation and adversarial conditions, highlighting the effectiveness of CRHNs for robust pattern retrieval tasks.
\end{itemize}

\subsection{Organization}

The remainder of this paper is organized as follows. Section~\ref{sec:literature_review} reviews related work on associative memory models and adversarial robustness. Section~\ref{sec:preliminary} introduces the background of MHNs and PCNs. Section~\ref{sec:crhn} describes the proposed CRHNs and the corresponding architecture. Meanwhile, it presents the formulation of RHNs, including their dynamical behavior and training via the Subspace Rotation Algorithm (SRA). Section~\ref{sec:simulation} provides experimental results, including evaluations under adversarial attacks and input degradations. Finally, Section~\ref{sec:conclusion} concludes the paper and discusses future research directions.

\section{Literature Review}\label{sec:literature_review}

\subsection{Auto-Associative Memory}
One of the earliest and most influential models in auto-associative memory is the HNNs, introduced by Hopfield in 1982 \cite{hopfield1982neural}. The HNNs formulate memory storage as an energy minimization process, where stored patterns correspond to attractors of the energy landscape. Despite its theoretical elegance and biological inspiration, the classical HNNs suffer from limited storage capacity, which scales as approximately $0.15N$ for $N$ neurons, as well as the presence of spurious attractors \cite{munakata2004hebbian, mceliece1987capacity}. To address these limitations, early efforts explored alternative learning rules and capacity-enhancing techniques, such as the perceptron-based training approach proposed by Gardner \cite{gardner1988space, gardner1988optimal}, as well as Hebbian learning variants and basin optimization strategies \cite{storkey1999basins}. However, only limited progress was achieved.

More recently, DAMs have significantly advanced the field by introducing higher-order interaction functions. Krotov et al. proposed exponential energy functions that reshape the energy landscape, enabling substantially increased storage capacity and improved pattern separation \cite{krotov2016dense, krotov2018dense}. These models theoretically achieve super-linear or even exponential memory capacity, marking a substantial departure from classical Hopfield networks. However, despite their theoretical advantages, DAMs are often difficult to train efficiently in practice and may not fully realize their capacity limits in real-world applications \cite{bao2022capacity}. Inspired by DAMs, MHNs were introduced, extending classical formulations to continuous state spaces and establishing a direct connection to attention mechanisms in Transformer architectures \cite{ramsauer2020hopfield}. In these models, memory retrieval can be performed in a single step through a softmax-based update rule, enabling efficient and differentiable associative recall. Nevertheless, in practice, they still exhibit limited robustness to noise and corrupted samples.

To improve retrieval selectivity and interpretability, recent studies have explored sparse and structured associative memory formulations. Sparse Hopfield networks introduce sparsity-inducing transformations to enforce more selective retrieval behavior \cite{hu2023sparse}. Furthermore, the Hopfield-Fenchel-Young framework provides a unified perspective that generalizes classical HNNs, DAMs, and MHNs through convex analysis and energy-based formulations \cite{santos2025hopfield}. This framework enables the design of new associative memory models with desirable properties such as sparsity, structured retrieval, and exact recovery guarantees.

Parallel to the development of MHNs, PCNs simulate memory retrieval as an iterative error minimization process across hierarchical layers, where each layer attempts to predict the activity of the layer below \cite{salvatori2021associative}. PCNs have demonstrated strong performance in memory reconstruction tasks and offer a biologically plausible interpretation of cortical processing. However, their reliance on iterative inference increases computational cost compared to the one-step retrieval mechanisms in MHNs \cite{tang2023recurrent, tscshantz2023hybrid}. Additionally, feature-space associative memory models improve upon traditional approaches by computing similarity in learned embedding spaces rather than raw input space, enabling more robust retrieval on complex data across different contexts \cite{salvatori2023feature}. Meanwhile, recent work on continuous attractor models shows that even when exact attractor structures are fragile, approximate dynamics—such as slow manifolds—can still provide stable and functionally meaningful memory representations over finite time scales \cite{sagodi2024back}. Together, these perspectives extend associative memory from discrete fixed-point retrieval to more flexible, representation-driven, and dynamical frameworks.

Coinciding with these developments, RHNs have been proposed as an extension of classical HNNs by incorporating hidden layers, similar in spirit to Restricted Boltzmann Machines (RBMs) \cite{yeap2021implementation, fischer2012introduction}. The introduction of hidden representations enhances the expressive power of the network, leading to improved storage capacity and robustness to noise \cite{lin2023basin, halabian2021enhanced, li2021global, lin2024subspace}. Building upon this idea, convolutional extensions further integrate feature extraction and memory retrieval, enabling associative memory models to operate effectively on high-dimensional data such as images.

\section{Preliminaries}\label{sec:preliminary}
In this section, we introduce the fundamental concepts and training methodologies of MHNs and PCNs to facilitate subsequent discussions.

\subsection{Modern Hopfield Networks}

MHNs extend the classical HNNs by introducing a continuous-state formulation with significantly enhanced storage capacity and retrieval performance. Unlike traditional HNNs that rely on quadratic energy functions, MHNs employ a softmax-based association mechanism, which is mathematically equivalent to the attention mechanism widely used in deep learning architectures \cite{ramsauer2020hopfield}.

MHNs store a set of patterns $\{\xi^{\mu}\}_{\mu=1}^{N}$ and retrieves relevant memories through a similarity-based matching process. Given a query vector $\mathbf{x} \in \mathbb{R}^{d}$ and a memory matrix $\mathbf{M} \in \mathbb{R}^{N \times d}$, the retrieval operation is defined as:

\begin{equation}\label{equ:mhn_update}
	\mathbf{x}' = \sum_{\mu=1}^{N} \text{softmax}\left(\beta \mathbf{x} \cdot \xi^{\mu} \right) \xi^{\mu}
\end{equation}

\noindent
where $\beta$ is a scaling parameter controlling the sharpness of the association, and $\xi^{\mu}$ represents the $\mu$-th stored memory pattern. The softmax function ensures that the retrieved pattern is a weighted combination of stored memories, where weights are determined by similarity.

This formulation can be equivalently written in matrix form:

\begin{equation}\label{equ:mhn_matrix}
	\mathbf{x}' = \text{softmax}(\beta \mathbf{x} \mathbf{M}^T)\mathbf{M}
\end{equation}

\noindent
which reveals that MHNs perform a form of attention over stored patterns. As $\beta \to \infty$, the model approaches nearest-neighbor retrieval, whereas smaller values of $\beta$ produce smoother interpolations among stored patterns.

From an energy-based perspective, MHNs can be derived from a continuous energy function:

\begin{equation}\label{equ:mhn_energy}
	E(\mathbf{x}) = -\frac{1}{\beta} \log \left( \sum_{\mu=1}^{N} \exp\left(\beta \mathbf{x} \cdot \xi^{\mu} \right) \right)
\end{equation}

\noindent
which generalizes the classical Hopfield energy function. The update rule in Equation \ref{equ:mhn_update} corresponds to performing gradient descent on this energy function, ensuring convergence to a stored memory pattern under appropriate conditions.

In practical implementations, MHNs is often integrated into encoder-decoder architectures, where input samples are first mapped into a latent space, followed by associative retrieval using the Hopfield mechanism, and finally reconstructed back to the input space. This latent-space formulation improves scalability and allows the model to handle high-dimensional data such as images.

\subsection{Predictive Coding Networks}

PCNs are inspired by information processing in the neocortex, where hierarchical structures are used to generate predictions and correct sensory inputs. This framework has been shown to be effective for representation learning and associative memory (AM) tasks \cite{salvatori2021associative, yoo2022bayespcn, tang2023recurrent}.

A multi-layer PCN consists of $L$ layers. The first layer corresponds to sensory inputs, the intermediate layers represent latent features extracted from the data, and the topmost layer provides top-down predictions to lower layers. Each layer contains value nodes, which represent neural activities, error nodes, which encode prediction discrepancies, and synaptic weights connecting adjacent layers.

The network minimizes the total prediction error across all layers, defined as
\begin{equation}\label{equ:obj_pcn}
	F = -\frac{1}{2} \sum_{l=1}^L \| \mathbf{e}^{(l)} \|_2^2,
\end{equation}
where the prediction error at layer $l$ is given by
\begin{equation}\label{equ:error_fun}
	\mathbf{e}^{(l)} = \mathbf{x}^{(l)} - \boldsymbol{\rho}^{(l)}.
\end{equation}

The prediction $\boldsymbol{\rho}^{(l)}$ is defined as
\begin{equation}\label{equ:pcn_prediction}
	\boldsymbol{\rho}^{(l)} =
	\begin{cases}
		\Theta^{(l)} f(\mathbf{x}^{(l+1)}) & \text{if } l < L, \\
		\mathbf{0} & \text{if } l = L,
	\end{cases}
\end{equation}
where $f(\cdot)$ is a nonlinear activation function, such as ReLU, $\tanh$, or sigmoid.

The synaptic weights $\Theta^{(l)}$ are updated to minimize the objective function using a Hebbian-like learning rule:
\begin{equation}\label{equ:hebbian_pcn}
	\Delta \Theta^{(l)} = \alpha \, \mathbf{e}^{(l)} f(\mathbf{x}^{(l+1)})^T,
\end{equation}
where $\alpha$ is the learning rate.

During inference, the value nodes are iteratively updated to minimize the objective function via gradient-based dynamics:
\begin{equation}\label{equ:pcn_inference}
	\Delta \mathbf{x}^{(l)} =
	\begin{cases}
		\beta \left( -\mathbf{e}^{(l)} + f'(\mathbf{x}^{(l)}) \Theta^{(l-1)^T} \mathbf{e}^{(l-1)} \right), & 1 < l < L, \\
		0, & l = 1,
	\end{cases}
\end{equation}
where $\beta$ is the integration step size.

\section{Convolutional Restricted Hopfield Networks}\label{sec:crhn}

\begin{figure*}[htbp]
    \centering
    \includegraphics[width=0.90\textwidth]{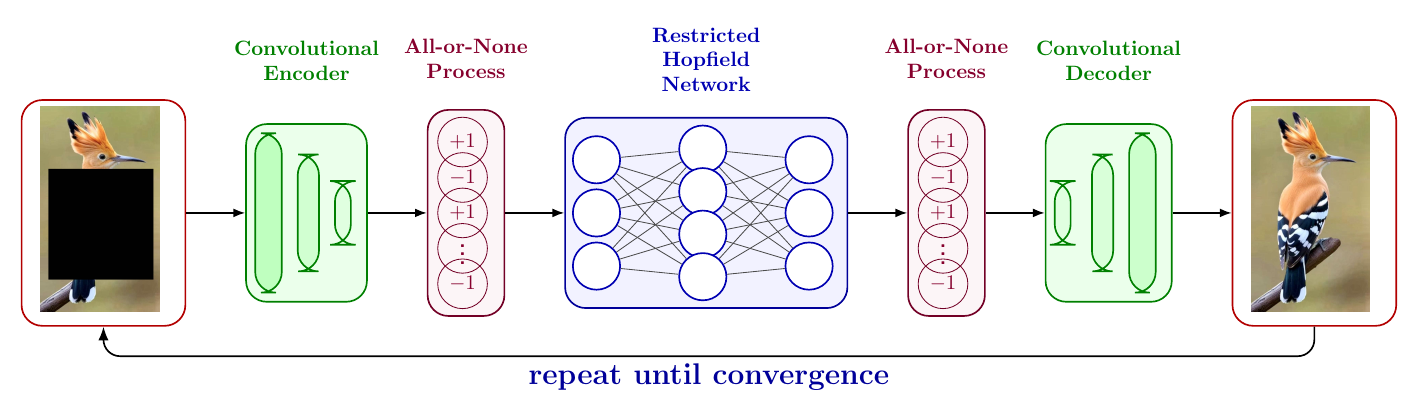}
    \caption{Architecture of Convolutional Restricted Hopfield Networks}
    \label{fig:arch_con_rhn}
\end{figure*}

As shown in Figure \ref{fig:arch_con_rhn}, the CRHNs are formulated as a recurrent dynamical system that integrates convolutional feature extraction into the associative memory. Given an input image, a convolutional encoder maps it into a latent representation, which is then projected onto a discrete interface via a binary activation function to emulate the all-or-none firing behavior of neurons \cite{RETTER2020116685}. The convolutional decoder maps the evolving state back to the image space at each iteration, producing progressively refined reconstructions. This recurrent formulation allows the CRHNs to jointly perform feature extraction, memory retrieval, and reconstruction within a unified iterative framework.

\subsection{Convolutional Encoder and Deconder}
The convolutional and deconvolutional operations in CRHNs form a weight-sharing encoder–decoder architecture that enables structured feature extraction and reconstruction. 

The convolutional encoder maps an input image \(x \in \mathbb{R}^{C \times H \times W}\) into a latent representation through a sequence of convolutional layers. For a single convolutional layer, the feature map is computed as
\[
z_{k}(i,j) = \phi \left( \sum_{c=1}^{C_{\text{in}}} \sum_{u=0}^{K-1} \sum_{v=0}^{K-1}
W_{k,c}(u,v)\, x_c(i+u, j+v) + b_k \right),
\]
where \(W_{k,c}\) denotes the convolutional kernel connecting input channel \(c\) to output channel \(k\), \(K\) is the kernel size, and \(\phi(\cdot)\) is a nonlinear activation function. Stacking multiple such layers yields the encoder mapping \(z = f_{\text{enc}}(x)\), which extracts hierarchical spatial features and compresses the input into a lower-dimensional representation.

To simulate the all-or-none firing behavior of neurons, the latent representation is discretized before entering the associative memory module:
\[
\tilde{z} = \mathrm{sign}(z),
\]
which serves as the interface between the convolutional encoder and the Hopfield network. This discretization introduces a binary decision boundary, while the subsequent Hopfield dynamics remain continuous.

The decoder reconstructs the input using transposed convolution (deconvolution), which expands the latent representation back to the original spatial resolution. For a single layer, the reconstruction is given by
\[
x'_{c}(i,j) = \sum_{k=1}^{C_{\text{out}}} \sum_{u=0}^{K-1} \sum_{v=0}^{K-1}
W_{k,c}(u,v)\, \tilde{z}_k(i-u, j-v),
\]
where the same kernel weights \(W_{k,c}\) are reused. In CRHNs, the encoder and decoder explicitly share the same weight matrix, i.e.,
\[
W_{\text{dec}} = W_{\text{enc}}^\top,
\]
so that the decoder uses the transpose of the encoder weights. This weight-sharing constraint enforces consistency between the encoding and reconstruction processes and reduces the number of free parameters and improve the stability of the dynamical system of CRHNs.

\subsection{Restricted Hopfield Networks}

Once the Convolutional Encoder encode the input images into a latent vector, the latent vector will be fed into the RHNs with \( K \) hidden layers. Without losing generality, let the first layer connected to the Convolutional module be considered as hidden layer 0, so the hidden layers can be indexed as \( h_0, h_1, \cdots, h_K \).

The weight matrix between any two interrelated layers is indexed as \( W_1, W_2, \cdots, W_K \), and the bias terms related to the layers are indexed by \( \theta_0, \theta_1, \cdots, \theta_K \). In the forward path, the signal before activation is represented by \( U \), which is indexed as \( U_1, U_2, \cdots, U_K \), and after the activation function is represented by \( H \), which is indexed as \( H_0, H_1, H_2, \cdots, H_K \). Please note that \( H_0 \) is actually the input signal. In the backward path, the reconstructed signal before activation is represented by \( R \), which is indexed as \( R_{K-1}, \cdots, R_0 \), and after the activation function is represented by \( V \), which is indexed as \( V_K, V_{K-1}, \cdots, V_1, V_0 \). Please note that \( V_K \) is actually \( H_K \).

Then, the dynamical behavior of the RHNs can be described as follows.

\textbf{In the Forward Path:}

\begin{equation}\label{equ:all_forward_path}
    \begin{aligned}
   &\ \frac{dU_{k}(t)}{dt} = W_{k} H_{k-1}(t) + \theta_{k}, \\
   &\ H_{k}(t) = g \odot \left(U_{k}(t)\right), \quad k = 1, 2, 3, \dots, K 
    \end{aligned}
\end{equation}

\noindent
where \( U_k(t) \) is the pre-activation state of layer \( k \), \( H_k(t) \) is the post-activation state, computed using a non-linear activation function \( g \odot (\cdot) \), which is an element-wise operation and is tanh in our study. \( W_k \) is the weight matrix connecting layer \( k-1 \) to \( k \), and \( \theta_k \) is the bias vector for layer \( k \). The state \( U_k(t) \) evolves dynamically over time as the input \( H_{k-1}(t) \) propagates through the network, producing the output \( H_k(t) \) for each layer.

\textbf{In the Backward Path:}

\begin{equation}\label{equ:all_backward_path}
    \begin{aligned}
    &\ \frac{dR_{k}(t)}{dt} = W^{T}_{k+1} V_{k+1}(t) + \theta_k, \\
    &\ V_{k}(t) = g \odot \left(R_{k}(t)\right), \quad k = K-1, \dots, 0
    \end{aligned}
\end{equation}

\noindent
where \( R_k(t) \) is the pre-activation reconstructed state of layer \( k \), \( V_k(t) \) is the post-activation reconstructed state, computed using \( g \odot (\cdot) \), which is an element-wise operation and is tanh in our study. \( W^{T}_{k+1} \) is the transpose of the weight matrix for backpropagating signals from layer \( k+1 \) to \( k \), and \( \theta_k \) is the bias vector for layer \( k \). The backward path evolves over time, ensuring that the reconstruction \( R_k(t) \) and \( V_k(t) \) align with the original data propagated in the forward pass.

\subsection{Stability Analysis}

\begin{proposition}[Lyapunov stability of RHN dynamics]
	\label{prop:rhn_stability}
	Consider Restricted Hopfield Networks (RHNs) with forward and backward dynamics defined in Equations~\ref{equ:all_forward_path} and \ref{equ:all_backward_path}. Define the energy function
	\begin{equation}\label{equ:energy_fun_rhn_prop}
		\begin{aligned}
			E(t) = & -\frac{1}{2} \sum_{k=1}^{K} H_{k-1}(t)^T W_k H_k(t) - \sum_{k=1}^{K} \theta_k^T H_k(t) \\
			& - \frac{1}{2} \sum_{k=0}^{K-1} V_{k+1}(t)^T W_{k+1}^T V_k(t) - \sum_{k=0}^{K-1} \theta_k^T V_k(t).
		\end{aligned}
	\end{equation}
	Then, under standard activation functions such as \(\tanh\), sigmoid, or ReLU, the energy function \(E(t)\) is non-increasing over time, i.e.,
	\[
	\frac{dE(t)}{dt} \leq 0.
	\]
	Consequently, the RHNs dynamics are Lyapunov stable, and the network converges to a set of equilibrium points corresponding to local minima of the energy function.
\end{proposition}

\begin{proof}
	Taking the time derivative of the energy function \(E(t)\), we obtain
	\begin{equation}\label{equ:energy_deri_prop}
		\frac{dE(t)}{dt} = \sum_{k=1}^{K} \frac{\partial E}{\partial H_k(t)} \frac{dH_k(t)}{dt} + \sum_{k=0}^{K-1} \frac{\partial E}{\partial V_k(t)} \frac{dV_k(t)}{dt}.
	\end{equation}
	
	Using the forward and backward dynamics defined in Equations~\ref{equ:all_forward_path} and \ref{equ:all_backward_path}, the derivative can be expanded as
	\begin{equation}\label{equ:derivative_energy_prop}
		\begin{aligned}
			\frac{dE(t)}{dt} 
			& = - \sum_{k=1}^{K} g'\!\left(U_k(t)\right) \left(\frac{dU_k(t)}{dt}\right)^2 \\
			& \quad - \sum_{k=0}^{K-1} g'\!\left(R_k(t)\right) \left(\frac{dR_k(t)}{dt}\right)^2.
		\end{aligned}
	\end{equation}
	
	Since the activation function \(g(\cdot)\) is chosen as \(\tanh\), sigmoid, or ReLU, its derivative satisfies
	\[
	g'(x) \geq 0 \quad \text{for all } x.
	\]
	Therefore, each term in \eqref{equ:derivative_energy_prop} is non-positive, which implies
	\[
	\frac{dE(t)}{dt} \leq 0.
	\]
	
	Hence, the energy function \(E(t)\) is monotonically non-increasing along the system trajectories. Because \(E(t)\) is bounded below, the dynamics converge to a stable equilibrium set corresponding to stationary points of the energy function.
\end{proof}

\subsection{Mathematical Formulation for Subspace Rotation Algorithm}

\begin{proposition}[Subspace Rotation for Optimal Weight Alignment in RHNs]\label{pro:sra_rhn}
	Let a RHN store \( p \) patterns of dimension \( m \), represented by
	\( Y \in \mathbb{R}^{m \times p} \). Assume the weight matrix
	\( W \in \mathbb{R}^{m \times p} \) is orthogonal, and the RHNs output are given by
	\begin{equation}\label{equ:ini_rhn}
		\hat{Y} = \phi(YW) W^T,
	\end{equation}
	where \( \phi(\cdot) \) is a nonlinear activation function applied elementwise.
	Then, aligning \( \hat{Y} \) to \( Y \) under an orthogonal transformation is equivalent to solving
	\begin{equation}\label{equ:optimization}
		\min_{Q^T Q = I_p} \| Y - \hat{Y} Q \|_F,
	\end{equation}
	whose optimal solution is \[Q^\star = U V^T,\]
	where \( U \Sigma V^T \) is the singular value decomposition (SVD) of \( \hat{Y}^T Y \).
\end{proposition}

\begin{proof}
	Expanding the objective in \eqref{equ:optimization} yields
	\begin{equation}\label{equ:deduction}
		\begin{aligned}
			\| Y - \hat{Y} Q \|_F^2
			&= \| Y \|_F^2 + \| \hat{Y} \|_F^2
			- 2 \operatorname{tr}(Q^T \hat{Y}^T Y).
		\end{aligned}
	\end{equation}
	Since the first two terms are independent of \( Q \), minimizing
	\( \| Y - \hat{Y} Q \|_F \) is equivalent to
	\begin{equation}\label{equ:trace}
		\max_{Q^T Q = I_p} \operatorname{tr}(Q^T \hat{Y}^T Y).
	\end{equation}
	
	Let the singular value decomposition of \( \hat{Y}^T Y \) be 
	\[\hat{Y}^T Y = U \Sigma V^T,\]
	where \( U, V \in \mathbb{R}^{p \times p} \) are orthogonal matrices and
	\( \Sigma = \operatorname{diag}(\sigma_1, \ldots, \sigma_p) \).
	Then,
	\begin{equation}
		\begin{aligned}
			\operatorname{tr}(Q^T \hat{Y}^T Y)
			&= \operatorname{tr}(Q^T U \Sigma V^T)
			= \operatorname{tr}(Z \Sigma)
			= \sum_{i=1}^{p} z_{ii} \sigma_i,
		\end{aligned}
	\end{equation}
	where \( Z = V^T Q^T U \) is also an orthogonal matrix. Since \( |z_{ii}| \le 1 \),
	\[
	\operatorname{tr}(Q^T \hat{Y}^T Y) \le \sum_{i=1}^{p} \sigma_i,
	\]
	with equality if and only if \( Z = I_p \), which implies \( Q = U V^T \).
	This completes the proof \cite{schonemann1966generalized}.
\end{proof}

Therefore, we derive the pseudocode for the SRA applied to RHNs, as presented in Algorithms~\ref{alg:left_sra} and~\ref{alg:right_sra}, referred to as the Left Subspace Rotation Algorithm (LSRA) and the Right Subspace Rotation Algorithm (RSRA), respectively.

\begin{algorithm}
	\caption{Left Subspace Rotation Algorithm}\label{alg:left_sra}
	\begin{algorithmic}
		\STATE \textbf{Input}: Samples $X$, number of layers $N$
		\STATE \textbf{Output}: Weight matrices $\{W_1, \dots, W_N\}$
		\STATE \textbf{Initialize}: Orthogonal weights $\{W_l\}_{l=1}^N$
		
		\FOR{$l \leftarrow N$ \textbf{to} $1$}
		\STATE Compute forward pre-activation $X_l$ at layer $l$
		\STATE Construct backward post-activation $Y_l$ from output to layer $l$
		\STATE $U, \Sigma, V \leftarrow \mathrm{SVD}\left( {X_l}^\top Y_{l} \right)$
		\STATE $W_l \leftarrow U V \, W_l$
		\ENDFOR
		
		\RETURN $\{W_l\}_{l=1}^N$
	\end{algorithmic}
\end{algorithm}

\begin{algorithm}
	\caption{Right Subspace Rotation Algorithm}\label{alg:right_sra}
	\begin{algorithmic}
		\STATE \textbf{Input}: Samples $X$, number of layers $N$
		\STATE \textbf{Output}: Weight matrices $\{W_1, \dots, W_N\}$
		\STATE \textbf{Initialize}: Orthogonal weights $\{W_l\}_{l=1}^N$
		
		\FOR{$l \leftarrow (N + 1)$ \textbf{to} $2$}
		\STATE Compute forward pre-activation $X_l$ at layer $l$
		\STATE Construct backward post-activation $Y_l$ from output to layer $l$
		\STATE $U, \Sigma, V \leftarrow \mathrm{SVD}\left( {X_l}^\top Y_l \right)$
		\STATE $W_{l-1} \leftarrow W_{l-1} \, U V$
		\ENDFOR
		
		\RETURN $\{W_l\}_{l=1}^N$
	\end{algorithmic}
\end{algorithm}

The proposed SRA update the network weights by aligning forward and backward representations at each layer. Specifically, for each layer, the forward pre-activation features and the backward reconstructed features are computed, and their correlation is captured through a cross-covariance matrix. A SVD is then applied to this matrix to extract an optimal orthogonal transformation. In the left subspace rotation algorithm, the weight matrix at each layer is updated by left-multiplying the orthogonal transformation, effectively rotating the output subspace of the layer. In contrast, the right subspace rotation algorithm updates the weights by right-multiplying the transformation, which adjusts the input subspace of the preceding layer. By iteratively applying these rotations across layers, the RHN converges to an optimal configuration that facilitates accurate pattern memorization.

\begin{figure}[htbp]
	\centering
	\includegraphics[width=0.90\linewidth]{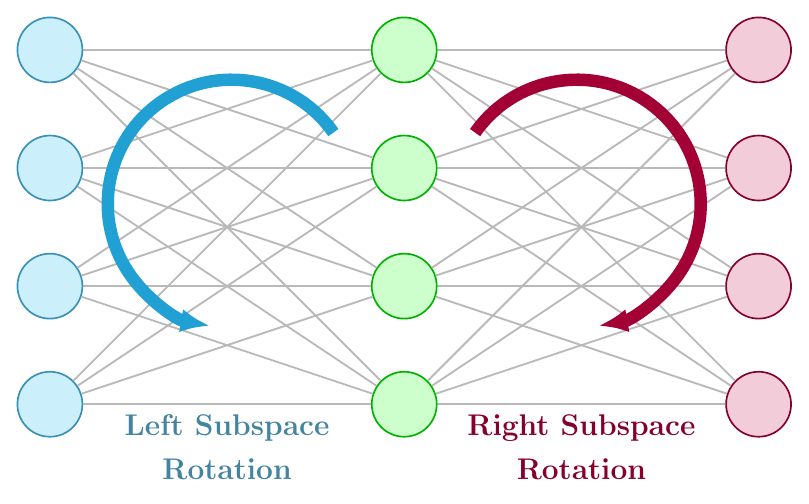}
	\caption{Illustration of Left Subspace Rotation (LSRA) and Right Subspace Rotation (RSRA) in updating the weight matrix of neural network. The orthogonal rotation matrices $U$ and $V$, obtained from the SVD process are applied to the weight matrices from the left or right side.}
	\label{fig:lsra_rsra}
\end{figure}

As shown in Figure \ref{fig:lsra_rsra}, the rotation matrices generated by the SVD process, namely $U$ and $V$ in Algorithm \ref{alg:left_sra} and \ref{alg:right_sra}, can be used to rotate the weight matrix from the left or the right. Empirical experiments indicate that when both LSRA and RSRA are applied in this context, the CRHN achieves better robustness than when only LSRA is used.

It is believed that although LSRA mainly improves the convergence behavior of the CRHN, RSRA rotates the weight matrix in a different orientation, causing the model in the nonlinear context to behave more like an orthogonal projector, which is beneficial for the robustness of the CRHNs, as illustrated in the next section.

\subsection{Robustness Analysis of Convolutional Restricted Hopfield Network}\label{sec:adversarial_mechanism}

\begin{proposition}[Robustness of semi-orthogonal weights to adversarial perturbations]\label{pro:orth_adv}
	Let $W \in \mathbb{R}^{d \times m}$ be a semi-orthogonal matrix such that $W^\top W = I_m$ .
	Consider the RHN retrieval operator $F(x) := \phi(WW^\top x)$, where $\phi(\cdot)$ is applied component-wise and is assumed to be $1$-Lipschitz with respect to the $\ell_2$ norm.
	
	Then, for any input vector $x \in \mathbb{R}^d$ and any perturbation $\delta \in \mathbb{R}^d$, we have
	\begin{equation}
		\|F(x+\delta)-F(x)\|_2 \le \|\delta\|_2 .
	\end{equation}
	
	Moreover, under the iterative RHN retrieval process
	\begin{equation}
		x^{(k+1)} = F(x^{(k)}), \qquad k \ge 0,
	\end{equation}
	the perturbation remains non-amplified after any number of retrieval steps. Specifically, for every integer $t \ge 1$,
	\begin{equation}
		\|F^t(x+\delta)-F^t(x)\|_2 \le \|\delta\|_2 ,
	\end{equation}
	where $F^t$ denotes applying $F$ repeatedly $t$ times.
	
	Therefore, semi-orthogonal weights constrain the RHN retrieval dynamics to be non-expansive, so adversarial perturbations are not enlarged during the retrieval process.
\end{proposition}

\begin{proof}
	Define
	\begin{equation}
		P = WW^\top .
	\end{equation}
	Because $W^\top W = I_m$, $P$ is the orthogonal projection matrix onto the column space of $W$, denoted by $\mathrm{span}(W)$. Hence,
	\begin{equation}
		\|P\|_2 = 1,
	\end{equation}
	and consequently, for any perturbation $\delta$,
	\begin{equation}
		\|P\delta\|_2 \le \|\delta\|_2 .
	\end{equation}
	
	Since $\phi$ is $1$-Lipschitz, for any two vectors $u,v \in \mathbb{R}^d$,
	\begin{equation}
		\|\phi(u)-\phi(v)\|_2 \le \|u-v\|_2 .
	\end{equation}
	
	Applying this property to $u=P(x+\delta)$ and $v=Px$, we obtain
	\begin{equation}
		\begin{aligned}
			\|F(x+\delta)-F(x)\|_2
			&= \|\phi(P(x+\delta))-\phi(Px)\|_2 \\
			&\le \|P(x+\delta)-Px\|_2 \\
			&= \|P\delta\|_2 \\
			&\le \|\delta\|_2 .
		\end{aligned}
	\end{equation}
	Thus, one step of RHN retrieval does not increase the perturbation magnitude.
	
	Since the mapping $F$ is non-expansive, composing it with itself preserves the same property. Therefore, by induction, for any $t \ge 1$,
	\begin{equation}
		\|F^t(x+\delta)-F^t(x)\|_2 \le \|\delta\|_2 .
	\end{equation}
	
	This proves that perturbations are not amplified during the iterative RHN retrieval dynamics.
\end{proof}

\begin{observation}[Noise filtering effect of low-dimensional subspaces]
	\label{observation:small_signal_subspace}
	Let $P_m = WW^\top$ be the orthogonal projector onto an $m$-dimensional subspace, where $W^\top W = I_m$.
	
	For isotropic perturbations satisfying $\mathbb{E}[\delta\delta^\top] = \sigma^2 I_d $, the projected perturbation magnitude becomes
	$\mathbb{E}\|P_m\delta\|_2^2 = \sigma^2 \mathrm{tr}(P_m) = \sigma^2 m$.
	
	Hence, a smaller latent subspace dimension $m$ reduces the expected perturbation magnitude propagated through the retrieval dynamics
	$x^{+}=\phi(P_m x)$.
\end{observation}

\section{Experiment and Discussion} \label{sec:simulation}

\subsection{Sample Preparation and Training Configuration}

All experiments are conducted on the Self-Taught Learning (STL) dataset \cite{coates2011analysis}, where images are resized to $96 \times 96$ and normalized to the range $[-1,1]$. From the dataset, we construct memory sets of varying sizes, including 50, 100, 250, and 500 images, to evaluate the scalability of different associative memory models. For each setting, the same subset of images is used across all models to ensure a fair comparison.

All models, including CRHNs, MHNs, and PCNs, are trained to memorize the selected samples using identical input representations. PCNs and MHNs follow their standard formulations. The MHNs are trained using the Adam optimizer with a learning rate of $1 \times 10^{-4}$. For PCNs, following \cite{salvatori2021associative}, we adopt the SGD optimizer with a learning rate of $1 \times 10^{-4}$ due to the use of estimated gradients. The training of CRHNs is divided into two stages. First, the encoder–decoder is trained to learn a compact representation of the input data, with shared weights to enforce structural consistency. Second, the latent representations are stored in the RHNs using the SRA, including LSRA and RSRA, which promotes semi-orthogonal weight structures. Each experiment is repeated five times with different random initializations to ensure statistical reliability.

For PCNs, we follow \cite{salvatori2021associative} and adopt an implementation adapted from a publicly available repository\footnote{PCNs: \url{https://github.com/C16Mftang/covariance-learning-PCNs}}. 
For MHNs, we follow \cite{ramsauer2020hopfield} and use the Hopfield layers framework\footnote{MHNs: \url{https://github.com/ml-jku/hopfield-layers}}.

\subsection{Evaluation Metrics}

\begin{figure*}[htbp]
	\centering
	\begin{subfigure}[t]{0.315\textwidth}
		\centering
		\includegraphics[width=\linewidth]{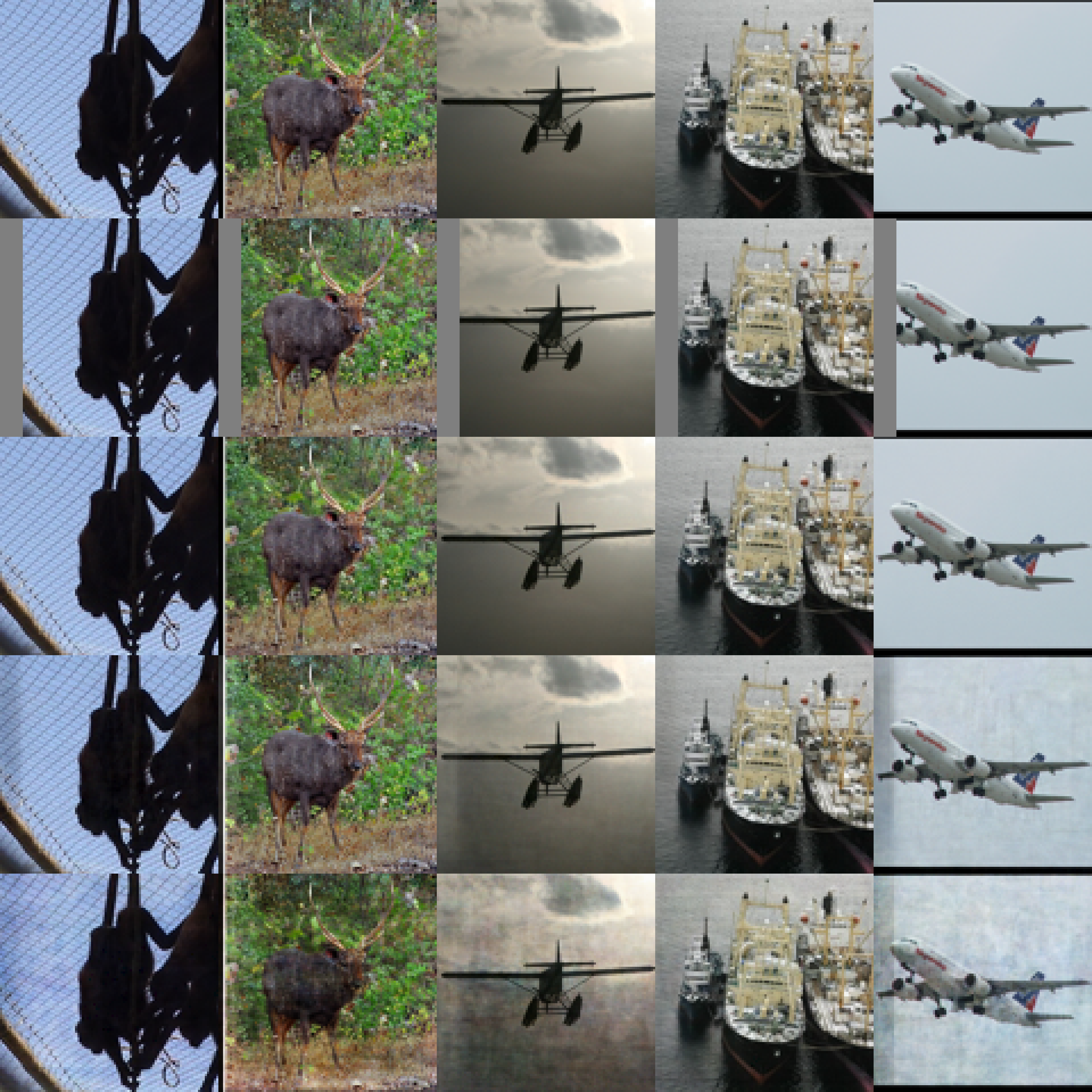}
		\caption{10 Columns are Corrupted}
		\label{fig:retrieval_dynamics_corrupt_10}
	\end{subfigure}
	\hfill
	\begin{subfigure}[t]{0.315\textwidth}
		\centering
		\includegraphics[width=\linewidth]{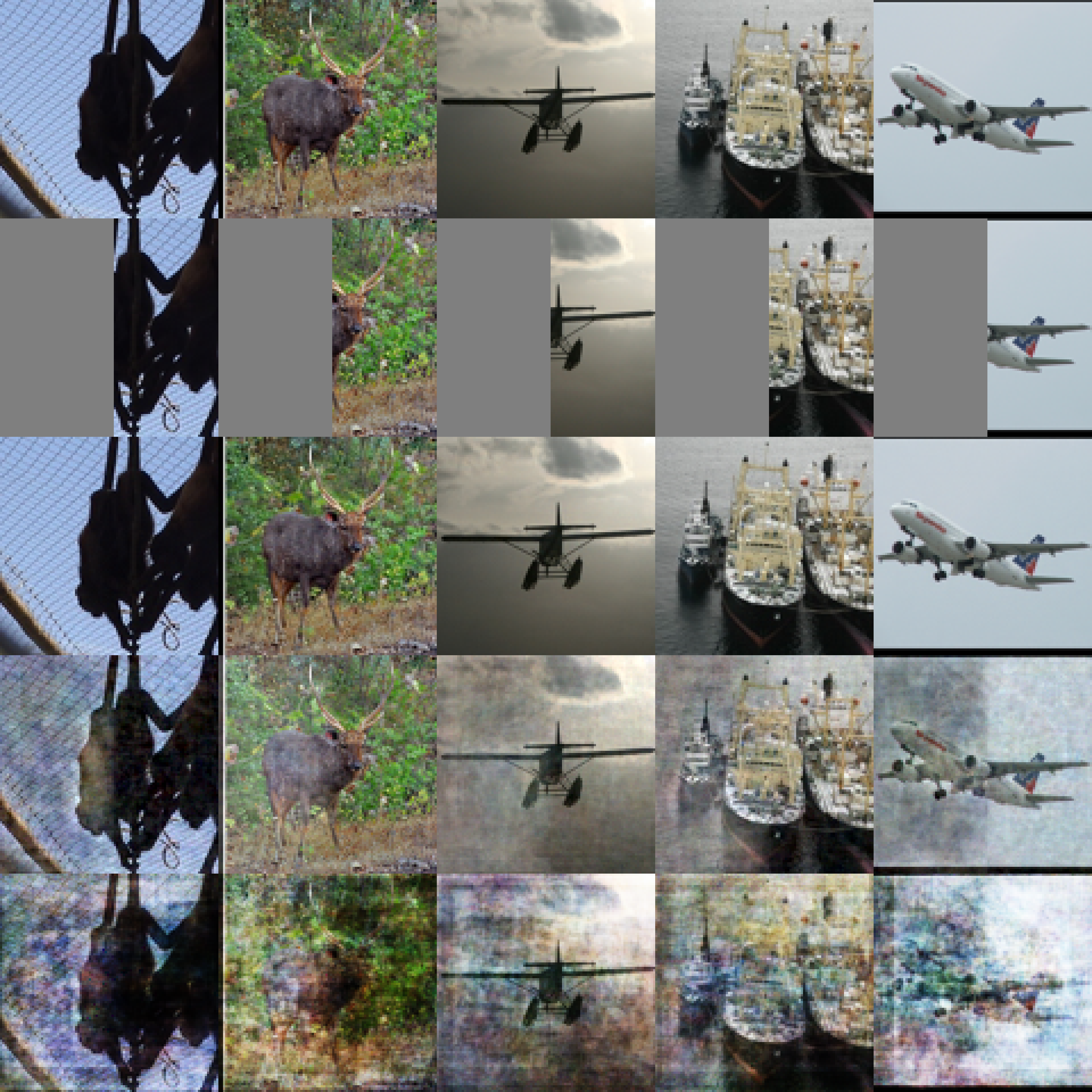}
		\caption{50 Columns are Corrupted}
		\label{fig:retrieval_dynamics_corrupt_50}
	\end{subfigure}
	\hfill
	\begin{subfigure}[t]{0.315\textwidth}
		\centering
		\includegraphics[width=\linewidth]{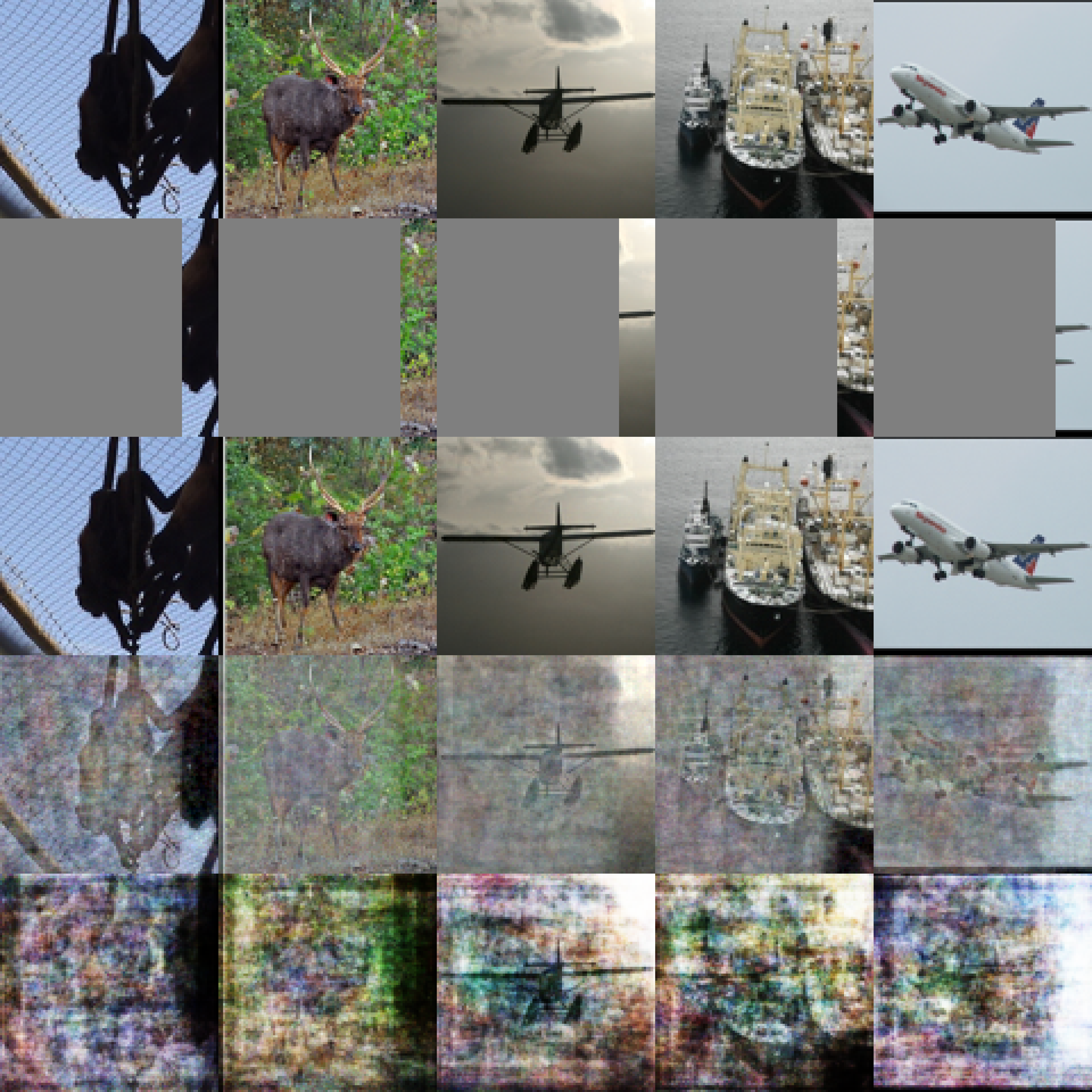}
		\caption{80 Columns are Corrupted}
		\label{fig:retrieval_dynamics_corrupt_80}
	\end{subfigure}
	\vspace{0.1em}
	\begin{subfigure}[t]{0.315\textwidth}
		\centering
		\includegraphics[width=\linewidth]{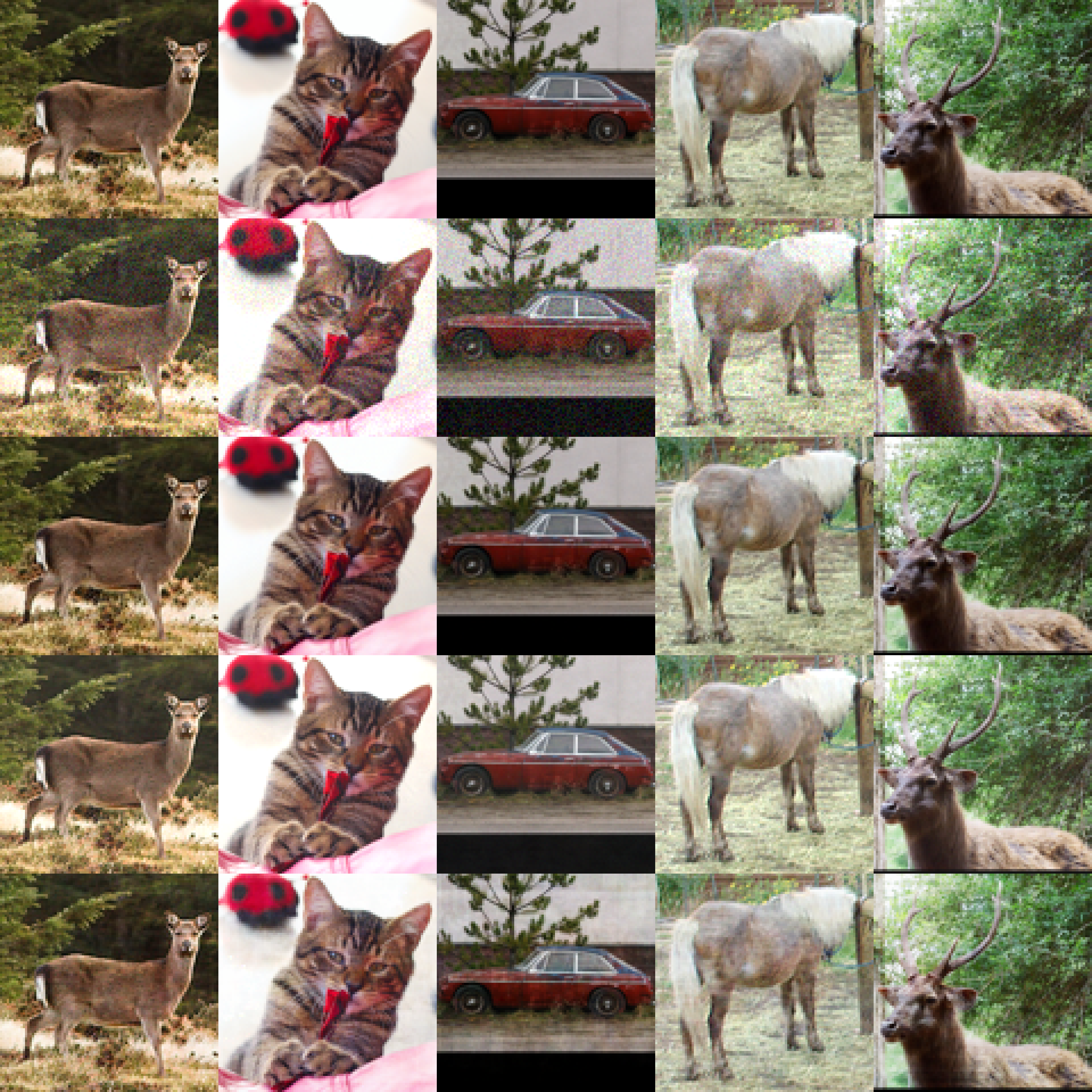}
		\caption{Gaussian Noise $\sigma=0.3$}
		\label{fig:retrieval_dynamics_noise_0.3}
	\end{subfigure}
	\hfill
	\begin{subfigure}[t]{0.315\textwidth}
		\centering
		\includegraphics[width=\linewidth]{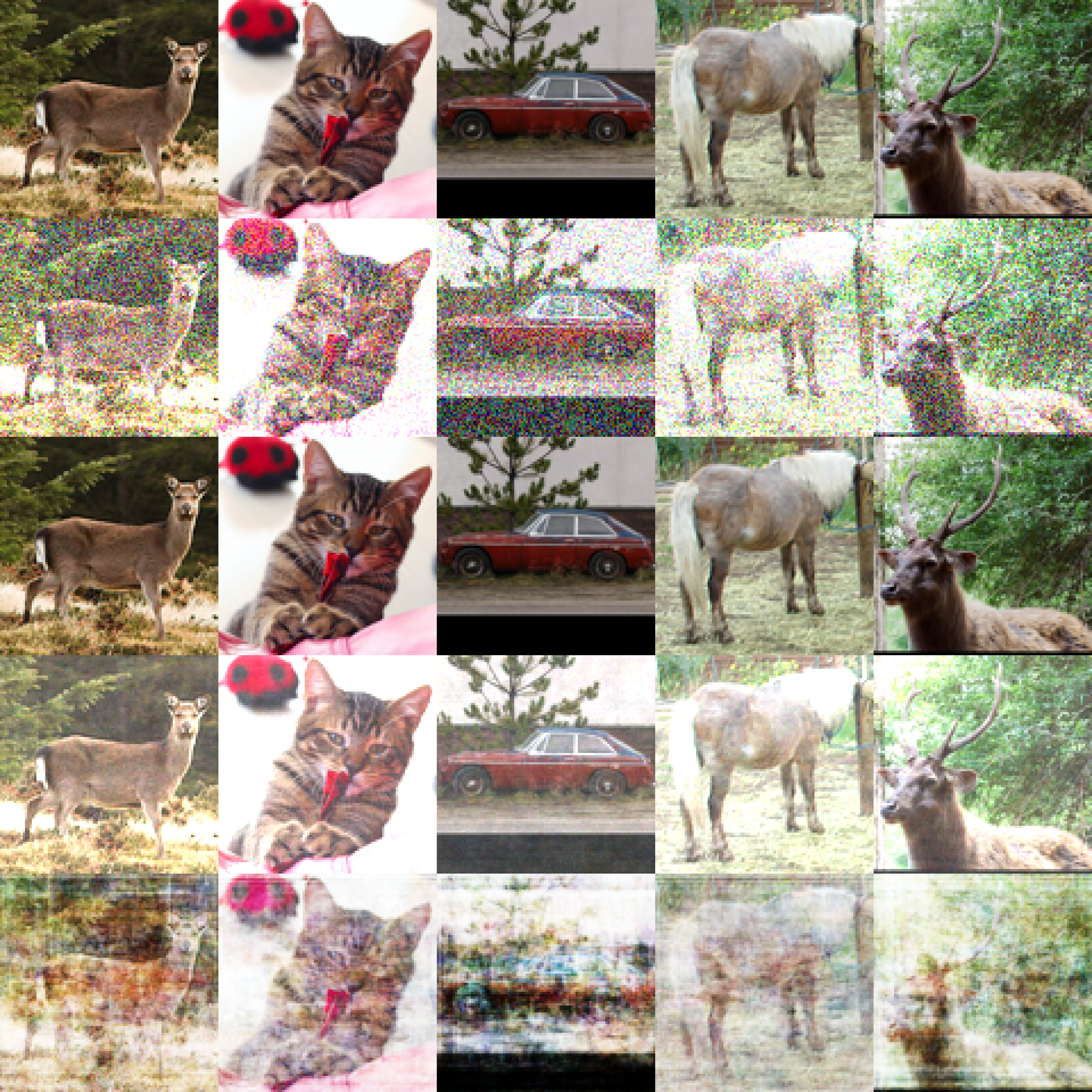}
		\caption{Gaussian Noise $\sigma=1.1$}
		\label{fig:retrieval_dynamics_noise_1.1}
	\end{subfigure}
	\hfill
	\begin{subfigure}[t]{0.315\textwidth}
		\centering
		\includegraphics[width=\linewidth]{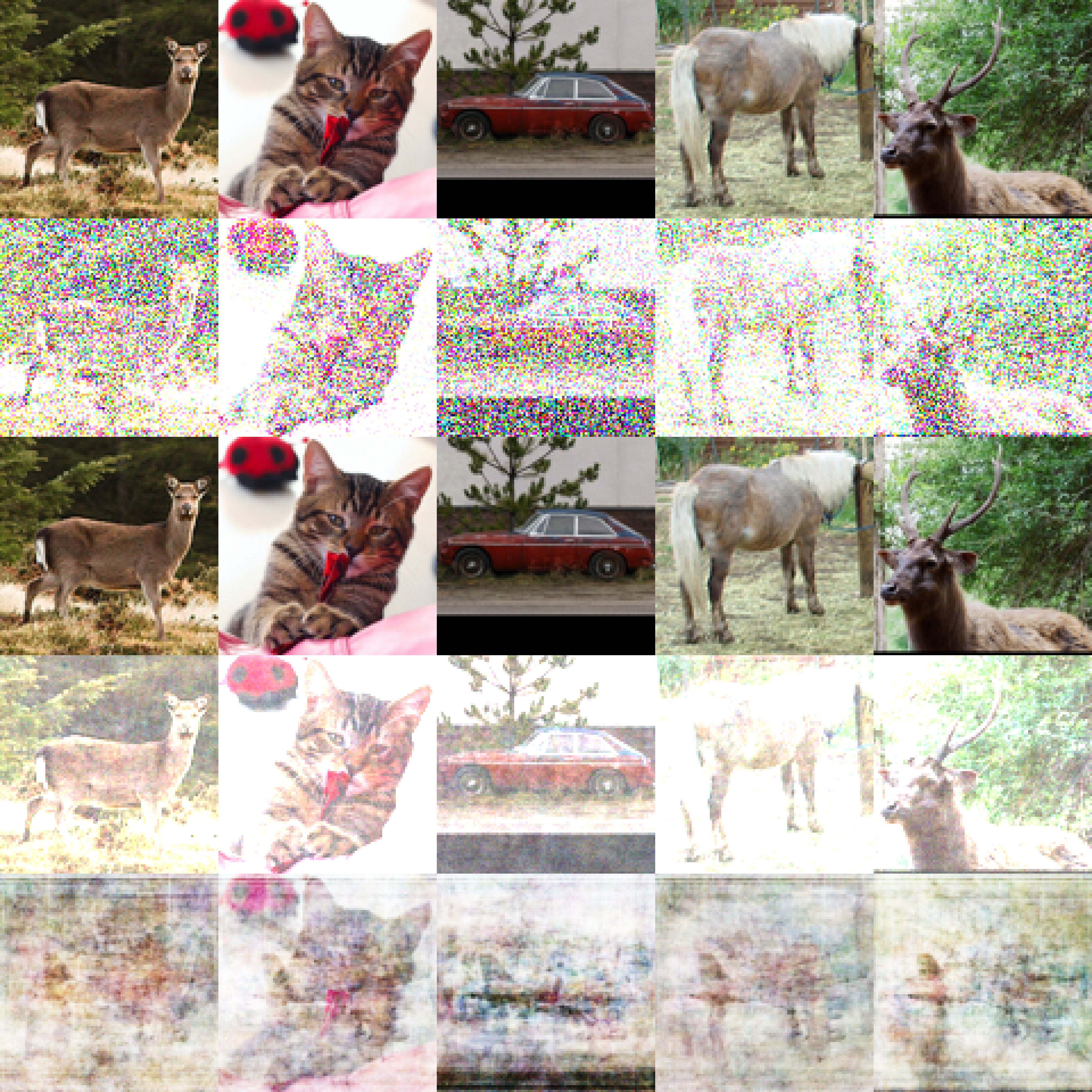}
		\caption{Gaussian Noise $\sigma=2.1$}
		\label{fig:retrieval_dynamics_noise_2.1}
	\end{subfigure}
	\caption{Qualitative retrieval comparison on the STL dataset, where all models are trained to memorize 250 images resized to $96\times96$. From top to bottom, the rows correspond to the original images, the corrupted and noisy inputs, and the retrieval results obtained by CRHNs, PCNs, and MHNs, respectively.}
	\label{fig:retrieval_dynamics_corrupt_noise}
\end{figure*}

\begin{figure*}[htbp]
	\centering
	\begin{subfigure}[t]{0.315\textwidth}
		\centering
		\includegraphics[width=\linewidth]{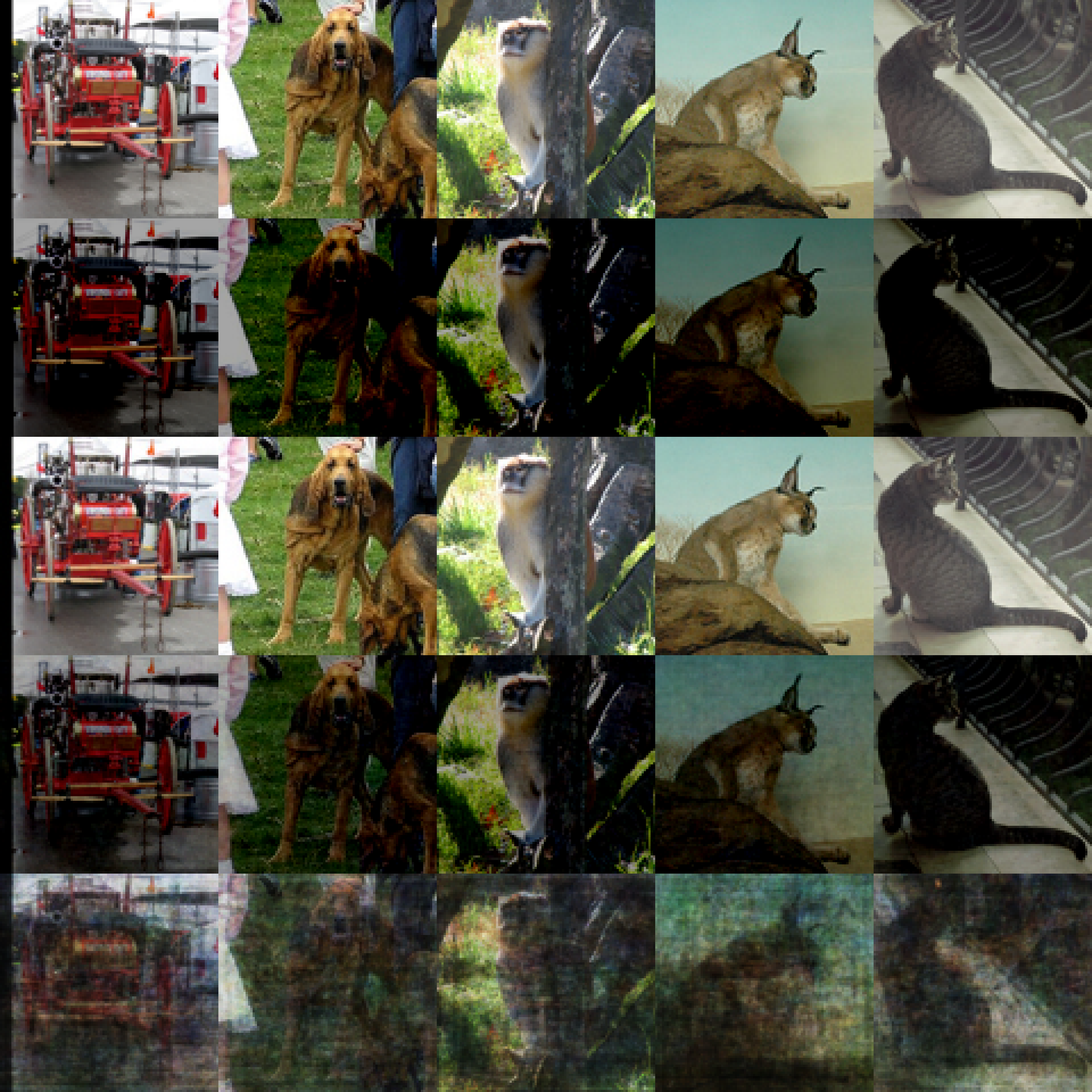}
		\caption{$\Delta = -0.7$}
		\label{fig:retrieval_dynamics_brightness_-0.7}
	\end{subfigure}
	\hfill
	\begin{subfigure}[t]{0.315\textwidth}
		\centering
		\includegraphics[width=\linewidth]{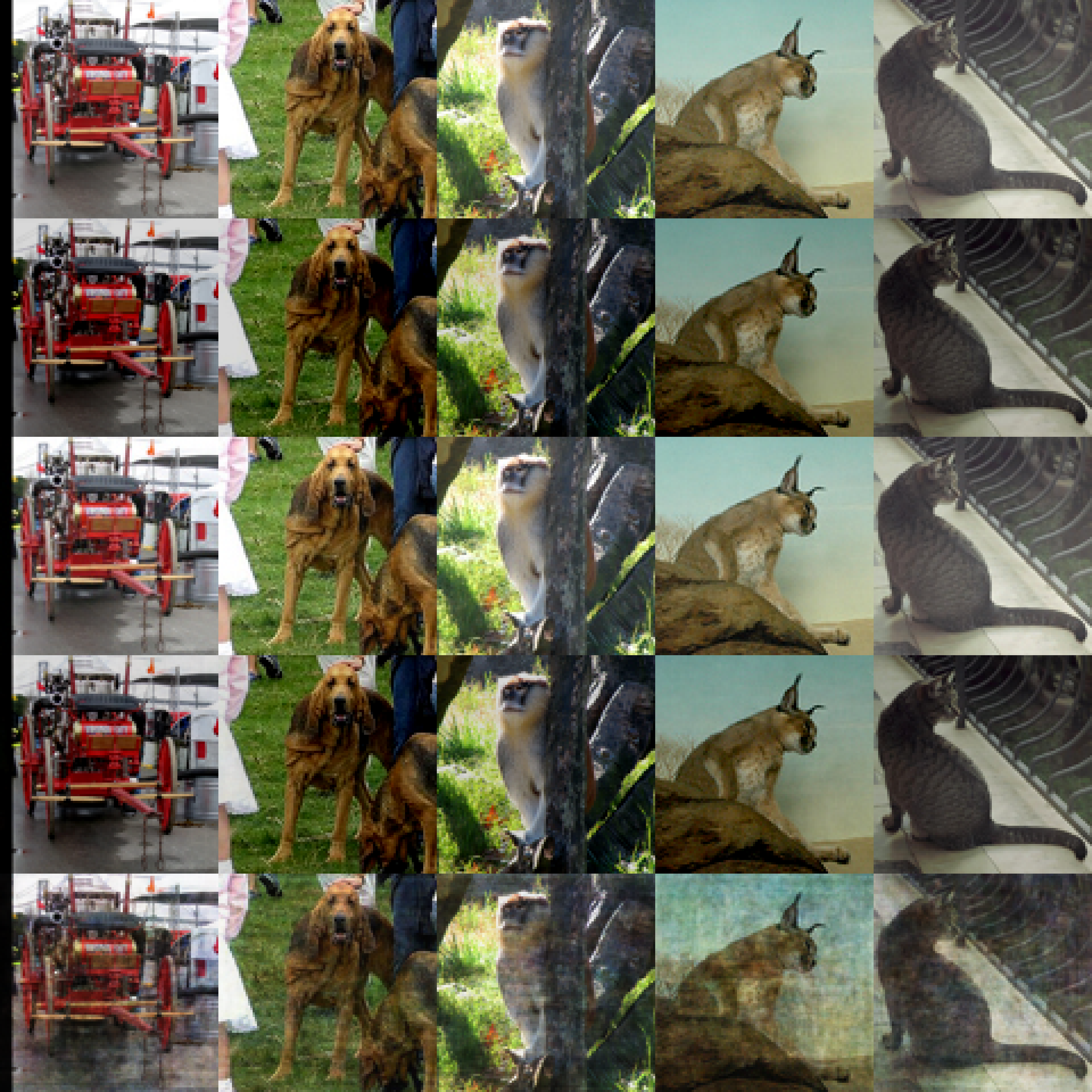}
		\caption{$\Delta = -0.3$}
		\label{fig:retrieval_dynamics_brightness_-0.3}
	\end{subfigure}
	\hfill
	\begin{subfigure}[t]{0.315\textwidth}
		\centering
		\includegraphics[width=\linewidth]{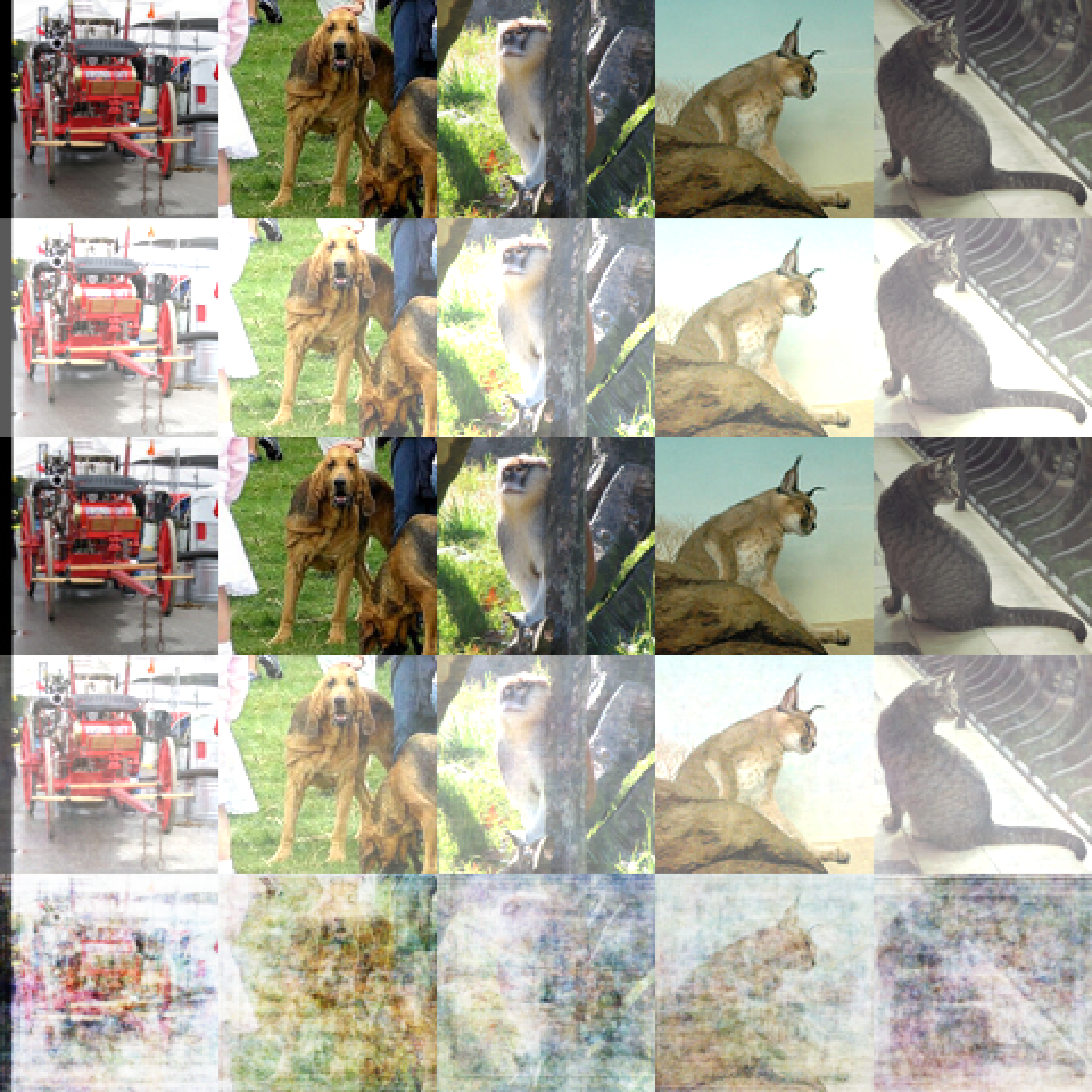}
		\caption{$\Delta = 0.7$}
		\label{fig:retrieval_dynamics_brightness_0.7}
	\end{subfigure}
	\vspace{0.1em}
	\begin{subfigure}[t]{0.315\textwidth}
		\centering
		\includegraphics[width=\linewidth]{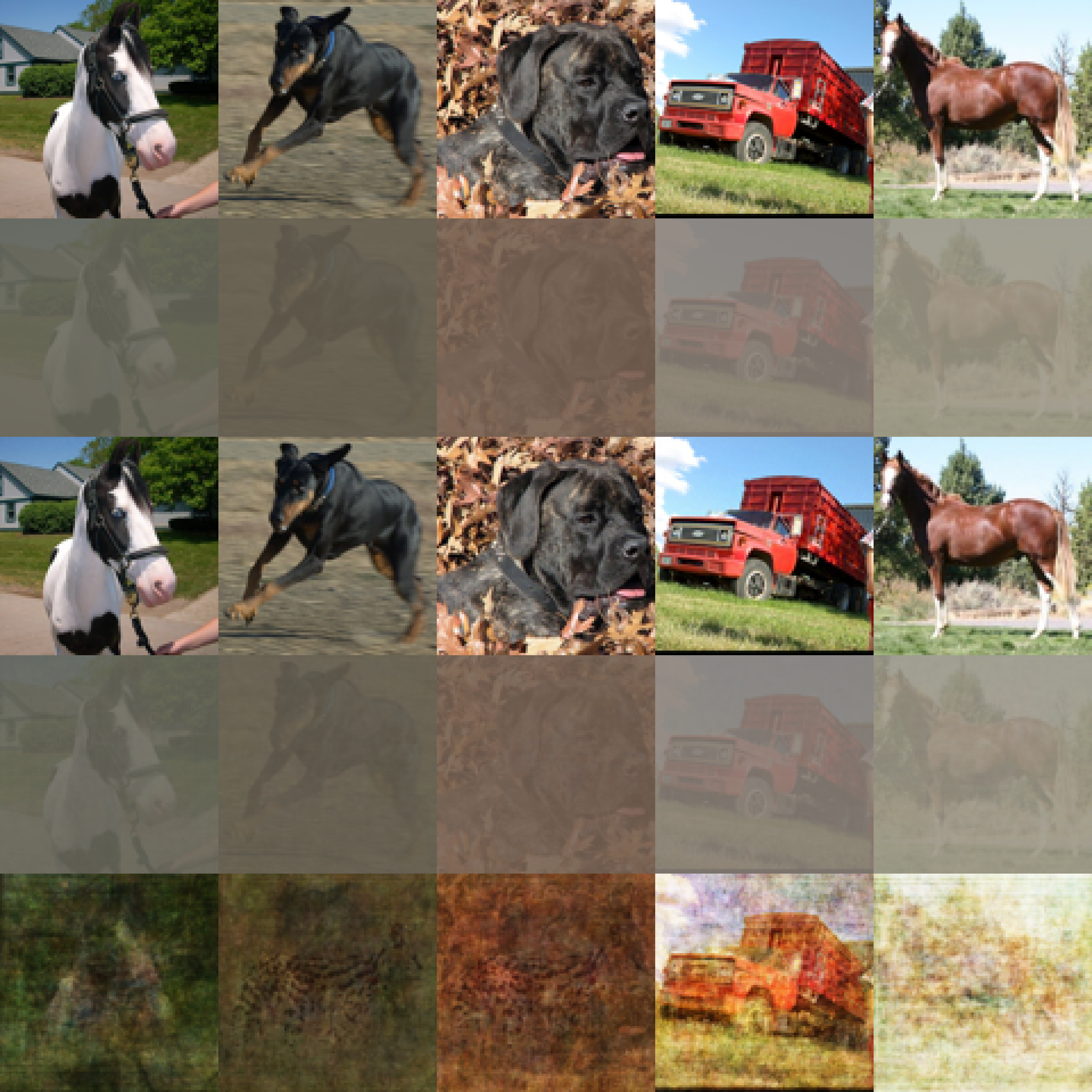}
		\caption{$\alpha = 0.1$}
		\label{fig:retrieval_dynamics_contrast_0.1}
	\end{subfigure}
	\hfill
	\begin{subfigure}[t]{0.315\textwidth}
		\centering
		\includegraphics[width=\linewidth]{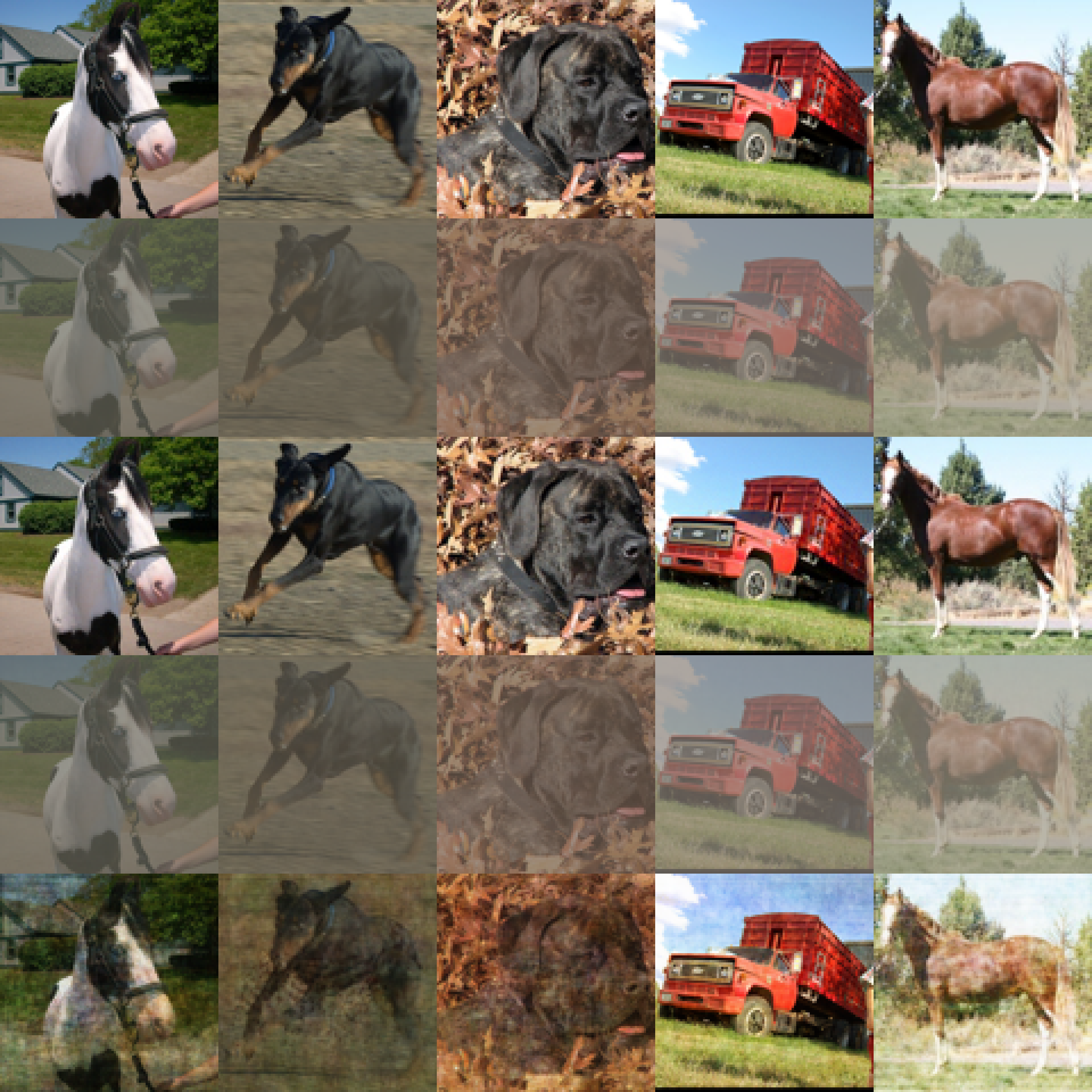}
		\caption{$\alpha = 0.3$}
		\label{fig:retrieval_dynamics_contrast_0.3}
	\end{subfigure}
	\hfill
	\begin{subfigure}[t]{0.315\textwidth}
		\centering
		\includegraphics[width=\linewidth]{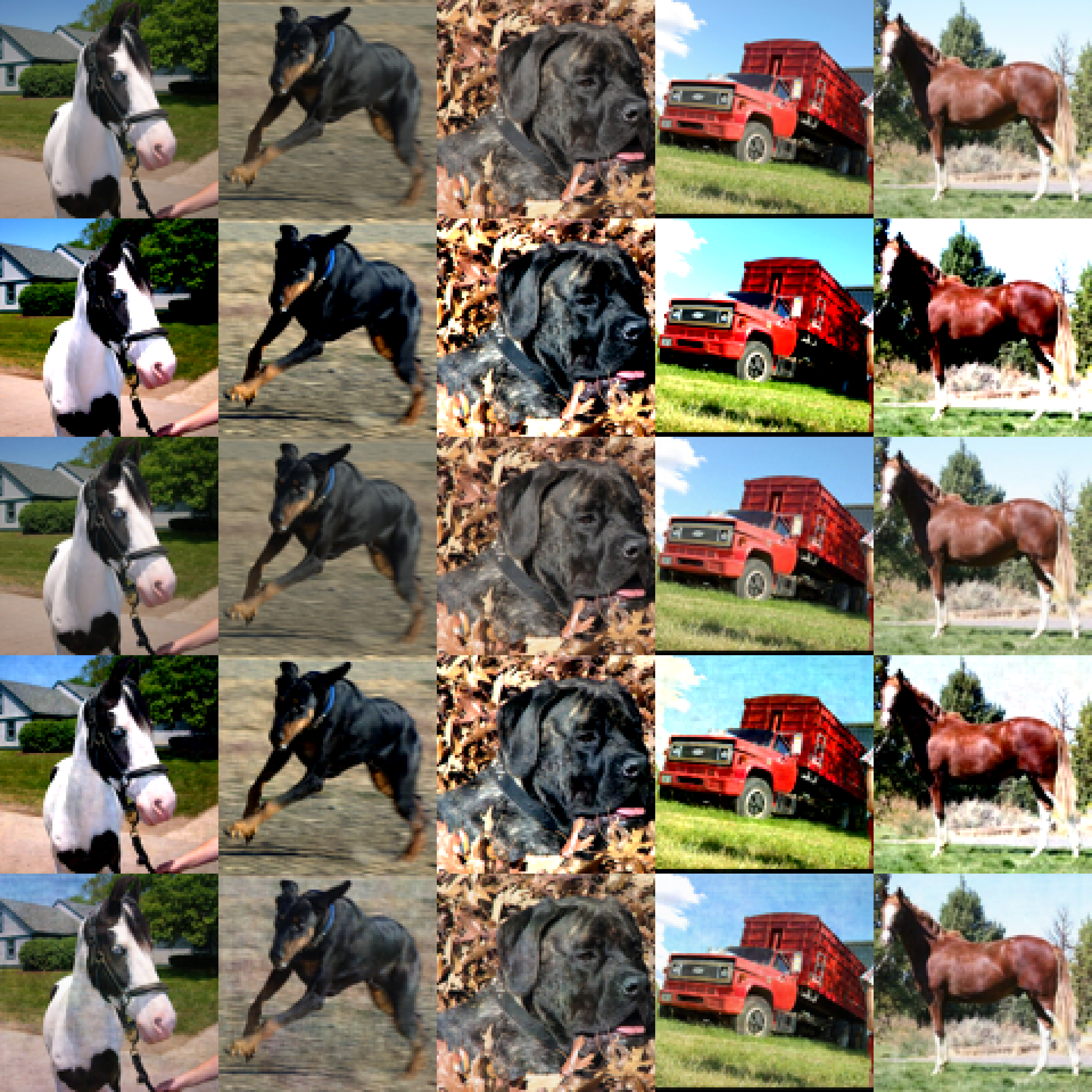}
		\caption{$\alpha = 1.9$}
		\label{fig:retrieval_dynamics_contrast_1.9}
	\end{subfigure}
	\caption{Qualitative retrieval comparison under brightness shift and contrast on the STL dataset. All models are trained to memorize 500 images resized to $96\times96$. Brightness shifts are applied in the normalized input space $[-1,1]$ by adding an offset $\Delta$ and clipping to the valid range. Contrast transformations are applied in the normalized input space $[-1,1]$ by scaling deviations $\alpha$ from the per-image mean. From top to bottom, the rows correspond to the original images, the brightness-shifted or contrast-transformed inputs, and the retrieval results obtained by  CRHNs, PCNs, and MHNs, respectively.}
	\label{fig:retrieval_dynamics_brightness_contrast}
\end{figure*}

We evaluate the auto-associative memory models under multiple input degradations to assess both reconstruction accuracy and robustness. The perturbations include photometric transformations, structural corruptions, and gradient-based adversarial variations.

Photometric perturbations are introduced to evaluate robustness to intensity variations. Specifically, brightness shifts are defined as
\begin{equation}
	x' = \mathrm{clip}(x + \Delta, -1, 1),
\end{equation}
and contrast shifts are defined as
\begin{equation}
	x' = \mathrm{clip}((x - \mu)\alpha + \mu, -1, 1),
\end{equation}
where $\mu$ denotes the per-image mean, $\Delta$ controls the brightness offset, and $\alpha$ adjusts the contrast. All perturbations are applied directly in the normalized space $[-1,1]$.

Structural corruptions are used to evaluate recovery from incomplete or noisy observations. These include partial masking and additive Gaussian noise, which test the model's ability to reconstruct missing information and suppress random perturbations.

In addition to these transformations, we evaluate robustness under gradient-based perturbations generated using a surrogate model. Specifically, we employ a convolutional autoencoder to produce adversarial inputs using methods such as the Fast Gradient Sign Method (FGSM) \cite{goodfellow2014explaining}, Fast FGSM (FFGSM) \cite{wong2020fast}, Basic Iterative Method (BIM) \cite{kurakin2018adversarial}, Projected Gradient Descent (PGD) \cite{mkadry2017towards}, Momentum Iterative FGSM (MI-FGSM) \cite{dong2018boosting}, Nesterov Iterative FGSM (NI-FGSM) \cite{lin2019nesterov}, Diverse Input FGSM (DI-FGSM) \cite{xie2019improving}, and Expectation Over Transformation PGD (EOTPGD) \cite{liu2018adv}.This approach enables consistent perturbation generation across different models, including RHNs, PCNs, and MHNs, whose recurrent and energy-based structures make direct gradient-based attacks less stable.

Performance is assessed quantitatively using the mean squared error (MSE) between the reconstructed output and the clean target image. All results are reported as mean $\pm$ standard deviation over five independent runs. Qualitatively, we visualize retrieval results under different perturbation settings to compare reconstruction fidelity and robustness across models.

\subsection{Retrieval Performance under Structured and Photometric Degradations}
The qualitative retrieval results under different input degradations are presented in Figures~\ref{fig:retrieval_dynamics_corrupt_noise} and \ref{fig:retrieval_dynamics_brightness_contrast}, including partial masking, additive noise, and brightness or contrast shifts. These results provide an intuitive comparison of the reconstruction behavior of CRHNs, PCNs, and MHNs under increasingly challenging conditions.

As shown in Figure~\ref{fig:retrieval_dynamics_corrupt_noise}, the models are evaluated under corrupted and noisy patterns. In Figure~\ref{fig:retrieval_dynamics_corrupt_10}, where only 10 columns are occluded, all models are able to retrieve the corrupted images; however, for PCNs and MHNs, a slight loss of accuracy is observed in the occluded regions. When 50 columns are occluded, as shown in Figure~\ref{fig:retrieval_dynamics_corrupt_50}, both PCNs and MHNs begin to struggle to reconstruct the images accurately. In contrast, CRHN is still able to recover the correct patterns, provided that the remaining visible regions are sufficiently distinctive. When the corruption becomes more severe, with 80 columns occluded (Figure~\ref{fig:retrieval_dynamics_corrupt_80}), PCNs and MHNs fail to retrieve the original images entirely. In comparison, CRHN is still capable of completing the images by leveraging the stored patterns in the network. This behavior is consistent with the mechanism of CRHNs, where stored patterns correspond to fixed points of the dynamical system.

As shown in Figure~\ref{fig:retrieval_dynamics_noise_0.3}, when Gaussian noise with $\mu = 0$ and $\sigma = 0.3$ is added to the samples, all models are able to retrieve the images accurately. However, as shown in Figure~\ref{fig:retrieval_dynamics_noise_1.1}, when the noise level increases to $\sigma = 1.1$, MHNs fail to reconstruct the images, and the outputs produced by PCNs become noticeably blurred. In contrast, CRHN is still able to retrieve the images with high fidelity. When the noise level further increases to $\sigma = 2.1$, as shown in Figure~\ref{fig:retrieval_dynamics_noise_2.1}, MHNs completely fail, and the outputs of PCNs become even more degraded compared to the case of $\sigma = 1.1$. In contrast, CRHNs can still accurately reconstruct all the images without any loss of accuracy.

As shown in Figure~\ref{fig:retrieval_dynamics_brightness_contrast}, the models are evaluated under photometric perturbations, including brightness and contrast shifts. Unlike structural corruptions, these perturbations preserve spatial structure while altering pixel intensity distributions, thereby testing the robustness of the models to global appearance variations.

As demonstrated in Figures~\ref{fig:retrieval_dynamics_brightness_-0.7}, \ref{fig:retrieval_dynamics_brightness_-0.3}, and \ref{fig:retrieval_dynamics_brightness_0.7}, negative brightness shifts ($\Delta = -0.7$ and $\Delta = -0.3$) make the input images significantly darker, while a positive brightness shift ($\Delta = 0.7$) makes them significantly brighter. Under these conditions, MHNs exhibit noticeable reconstruction errors when the images are too dark or too bright. PCNs are unable to recover the original brightness and tend to preserve the intensity of the input images; however, their final reconstructions lose fine details. In contrast, CRHN is able to retrieve the patterns with high fidelity.

For contrast shifts, as shown in Figures~\ref{fig:retrieval_dynamics_contrast_0.1}, \ref{fig:retrieval_dynamics_contrast_0.3}, and \ref{fig:retrieval_dynamics_contrast_1.9}, reducing the contrast degrades the retrieval performance of MHNs, while increasing the contrast has a relatively smaller effect on their reconstruction accuracy. Similarly, PCNs fail to recover the original contrast and tend to preserve the contrast of the input images. In contrast, CRHN is able to accurately reconstruct the images within a reasonable range of contrast variations.

\subsection{Robustness under Gradient-Based Perturbations}

\begin{table*}[htbp]
	\centering
	\setlength{\tabcolsep}{4pt}
	\renewcommand{\arraystretch}{0.9}
	
	\begin{threeparttable}
		\caption{Reconstruction error (MSE) of associative memory models under various adversarial attack methods on STL images. Results are reported for memory sizes of 50 and 100. Lower values indicate better robustness.}
		\label{tab:attack_eval_50_100}
		
		\begin{tabular}{lccc@{\hspace{1.2em}}ccc}
			\toprule
			& \multicolumn{3}{c}{Memory Size = 50} 
			& \multicolumn{3}{c}{Memory Size = 100} \\
			\cmidrule(lr){2-4} \cmidrule(lr){5-7}
			Attack & MHN & PCN & CRHN & MHN & PCN & CRHN \\
			\midrule
			
			FGSM & $0.1242\!\pm\!0.0359$ & $0.1170\!\pm\!0.0222$ & $\mathbf{0.0337\!\pm\!0.0303}$ 
			& $0.0962\!\pm\!0.0186$ & $0.1123\!\pm\!0.0110$ & $\mathbf{0.0050\!\pm\!0.0056}$ \\
			
			FFGSM & $0.0559\!\pm\!0.0124$ & $0.0786\!\pm\!0.0120$ & $\mathbf{0.0002\!\pm\!0.0001}$ 
			& $0.0619\!\pm\!0.0072$ & $0.0815\!\pm\!0.0105$ & $\mathbf{0.0005\!\pm\!0.0005}$ \\
			
			DIFGSM & $0.1633\!\pm\!0.0211$ & $0.1509\!\pm\!0.0139$ & $\mathbf{0.0454\!\pm\!0.0167}$ 
			& $0.1744\!\pm\!0.0083$ & $0.1582\!\pm\!0.0091$ & $\mathbf{0.0214\!\pm\!0.0273}$ \\
			
			MIFGSM & $0.1430\!\pm\!0.0357$ & $0.1326\!\pm\!0.0182$ & $\mathbf{0.0362\!\pm\!0.0398}$ 
			& $0.1337\!\pm\!0.0243$ & $0.1315\!\pm\!0.0156$ & $\mathbf{0.0085\!\pm\!0.0116}$ \\
			
			NIFGSM & $0.1449\!\pm\!0.0452$ & $0.1398\!\pm\!0.0182$ & $\mathbf{0.0416\!\pm\!0.0265}$ 
			& $0.1410\!\pm\!0.0267$ & $0.1363\!\pm\!0.0173$ & $\mathbf{0.0099\!\pm\!0.0121}$ \\
			
			BIM & $0.1431\!\pm\!0.0389$ & $0.1326\!\pm\!0.0185$ & $\mathbf{0.0382\!\pm\!0.0429}$ 
			& $0.1340\!\pm\!0.0230$ & $0.1320\!\pm\!0.0166$ & $\mathbf{0.0056\!\pm\!0.0065}$ \\
			
			EOTPGD & $0.1171\!\pm\!0.0200$ & $0.1266\!\pm\!0.0170$ & $\mathbf{0.0260\!\pm\!0.0230}$ 
			& $0.1298\!\pm\!0.0168$ & $0.1300\!\pm\!0.0165$ & $\mathbf{0.0068\!\pm\!0.0083}$ \\
			
			PGD & $0.1163\!\pm\!0.0224$ & $0.1250\!\pm\!0.0188$ & $\mathbf{0.0206\!\pm\!0.0211}$ 
			& $0.1254\!\pm\!0.0159$ & $0.1302\!\pm\!0.0160$ & $\mathbf{0.0064\!\pm\!0.0091}$ \\
			
			\bottomrule
		\end{tabular}
		
		\begin{tablenotes}
			\footnotesize
			\item[1] Adversarial samples are generated using a surrogate convolutional autoencoder trained on STL images. Perturbations are computed with $\epsilon=0.6$ and evaluated consistently across all models.
		\end{tablenotes}
		
	\end{threeparttable}
\end{table*}

\begin{figure*}[t]
	\centering
	\begin{subfigure}[t]{0.24\textwidth}
		\centering
		\includegraphics[width=\linewidth]{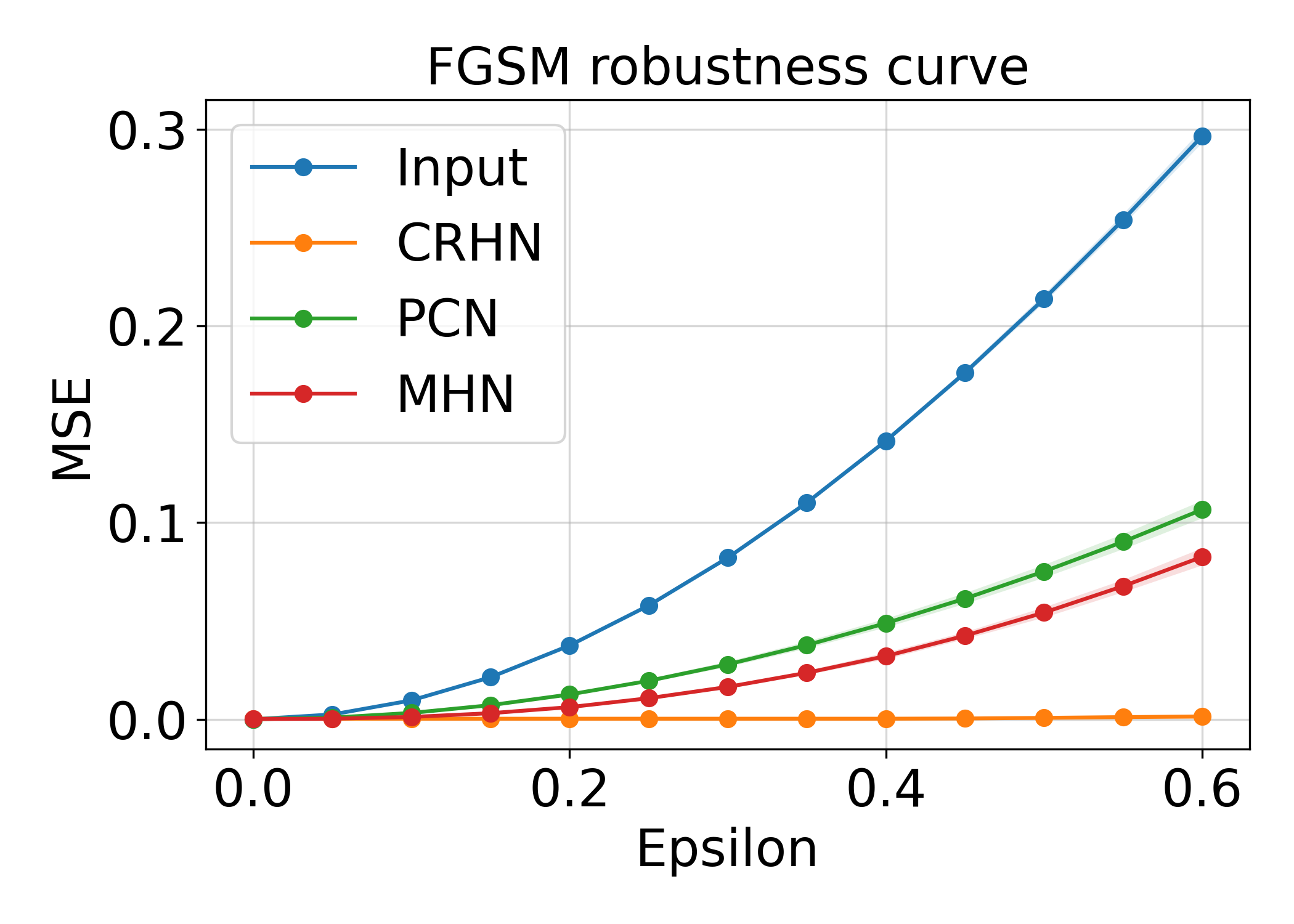}
	\end{subfigure}
	\hfill
	\begin{subfigure}[t]{0.24\textwidth}
		\centering
		\includegraphics[width=\linewidth]{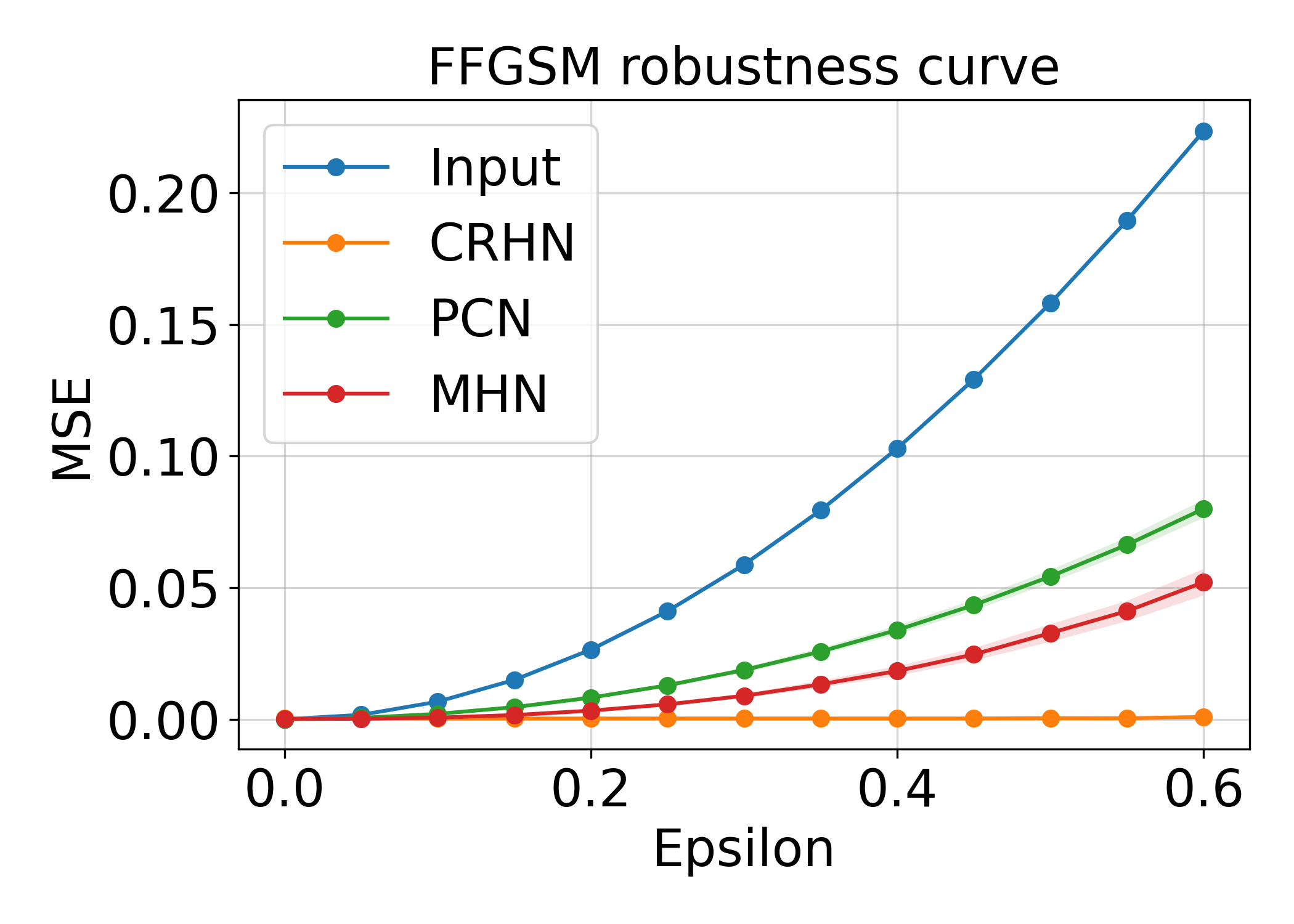}
	\end{subfigure}
	\hfill
	\begin{subfigure}[t]{0.24\textwidth}
		\centering
		\includegraphics[width=\linewidth]{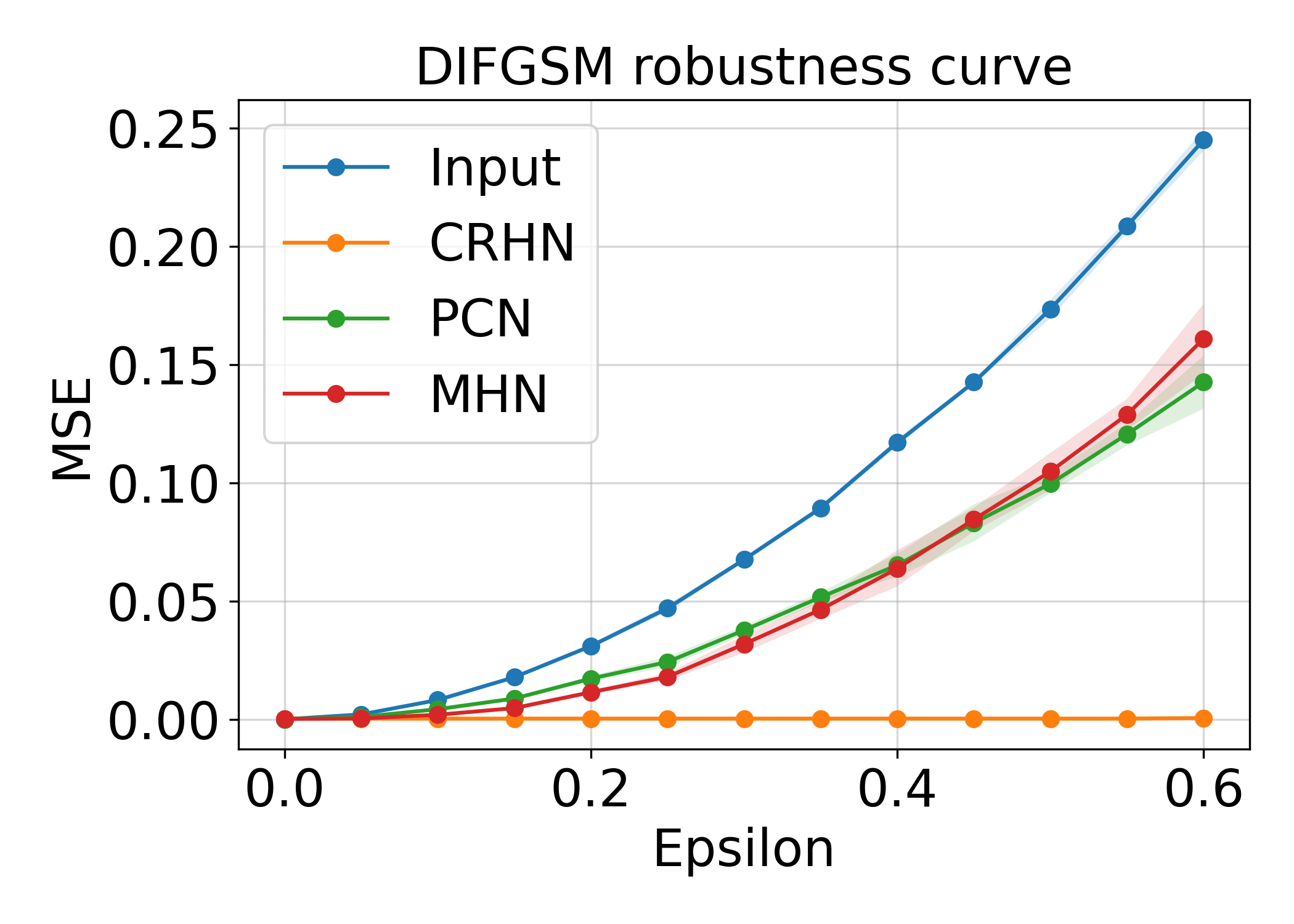}
	\end{subfigure}
	\hfill
	\begin{subfigure}[t]{0.24\textwidth}
		\centering
		\includegraphics[width=\linewidth]{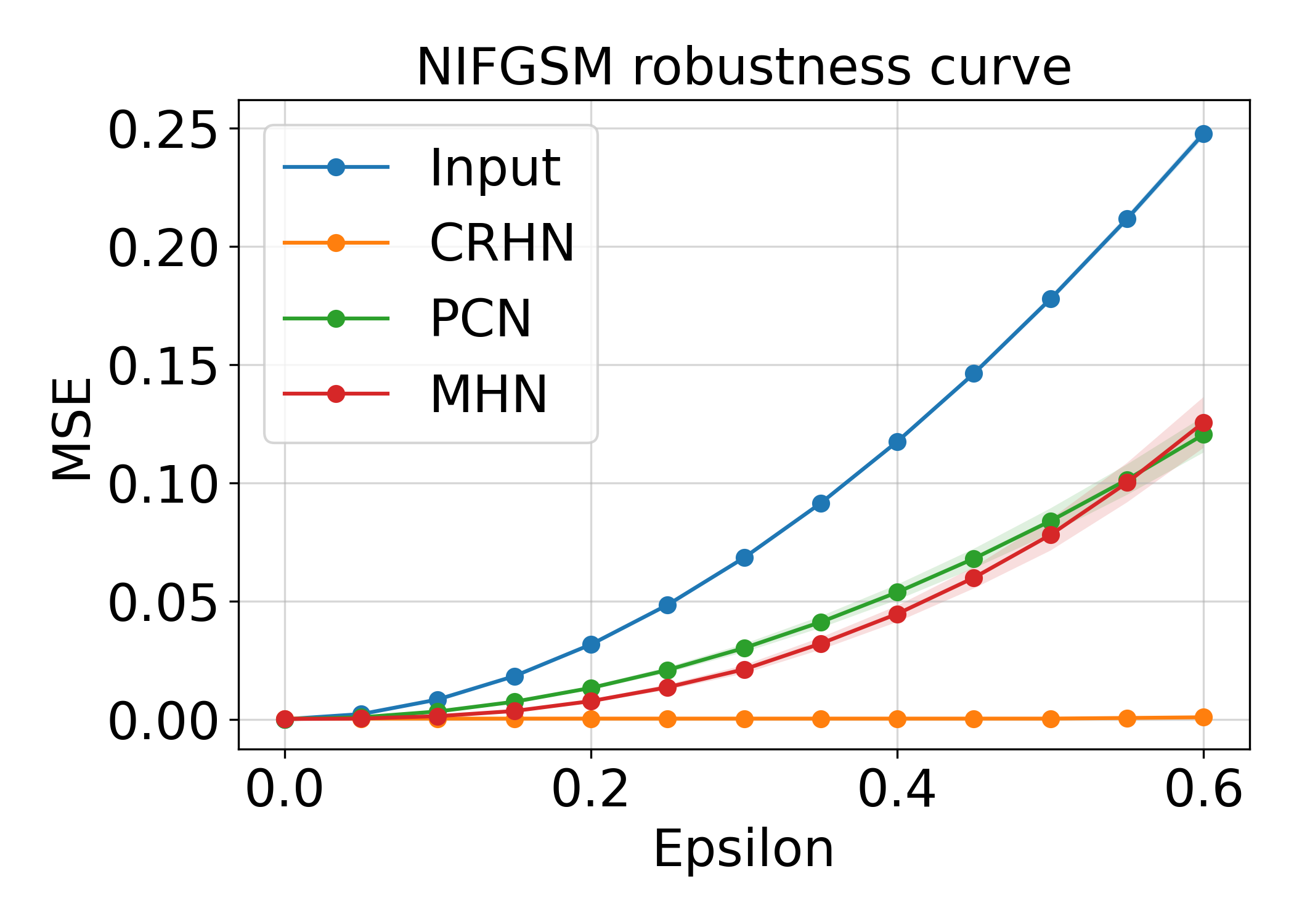}
	\end{subfigure}
	\hfill
	\begin{subfigure}[t]{0.24\textwidth}
		\centering
		\includegraphics[width=\linewidth]{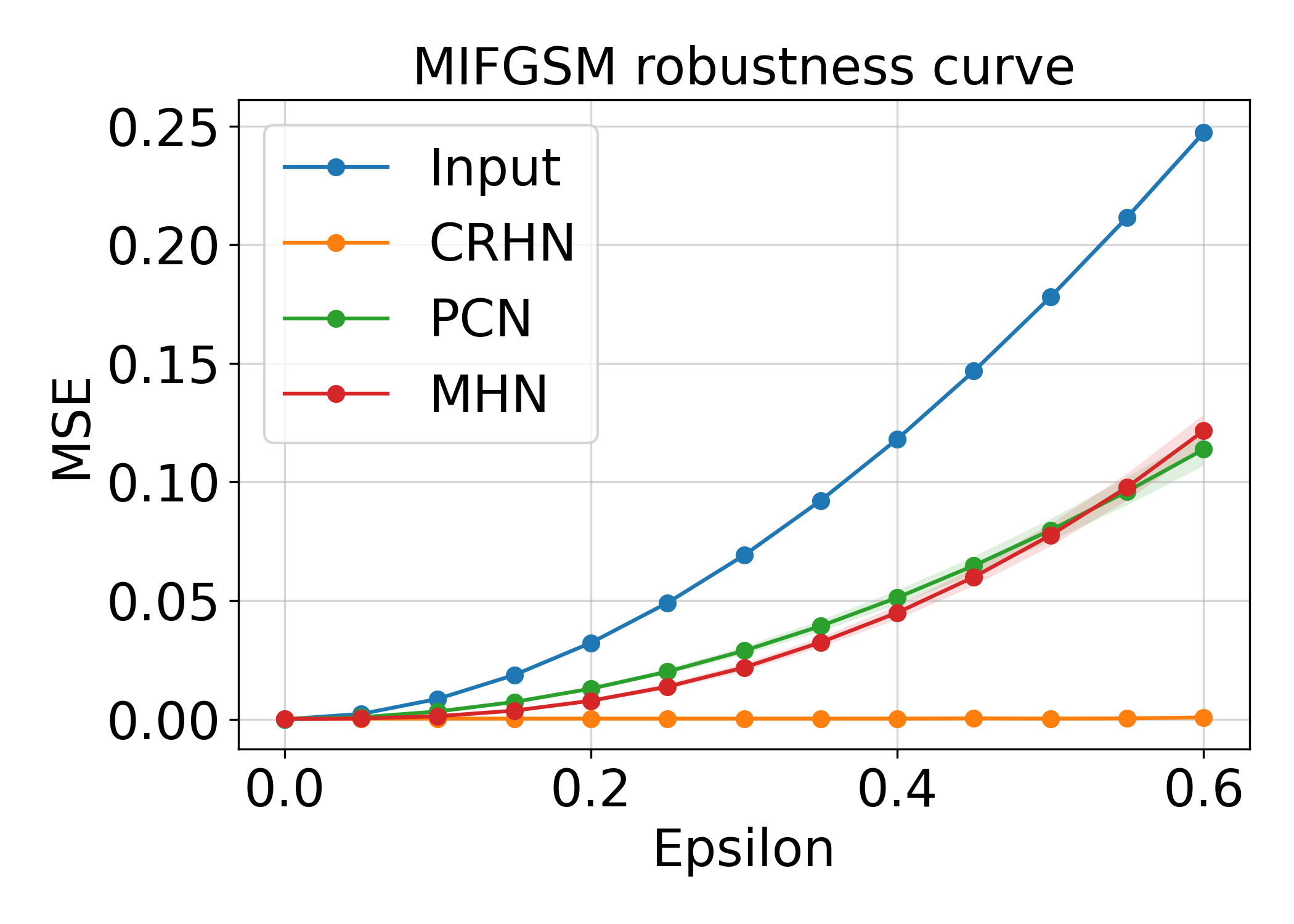}
	\end{subfigure}
	\hfill
	\begin{subfigure}[t]{0.23\textwidth}
		\centering
		\includegraphics[width=\linewidth]{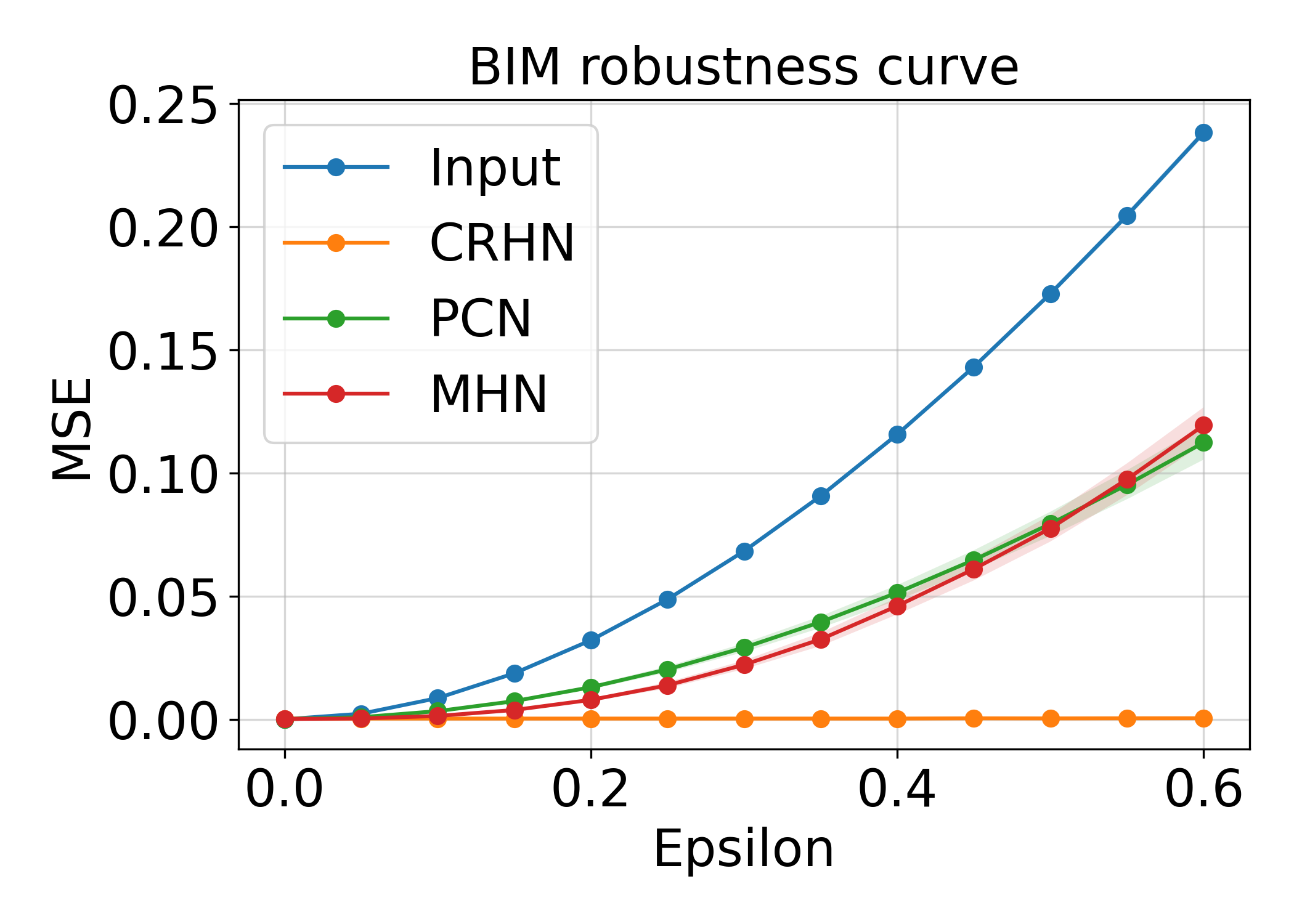}
	\end{subfigure}
	\hfill
	\begin{subfigure}[t]{0.24\textwidth}
		\centering
		\includegraphics[width=\linewidth]{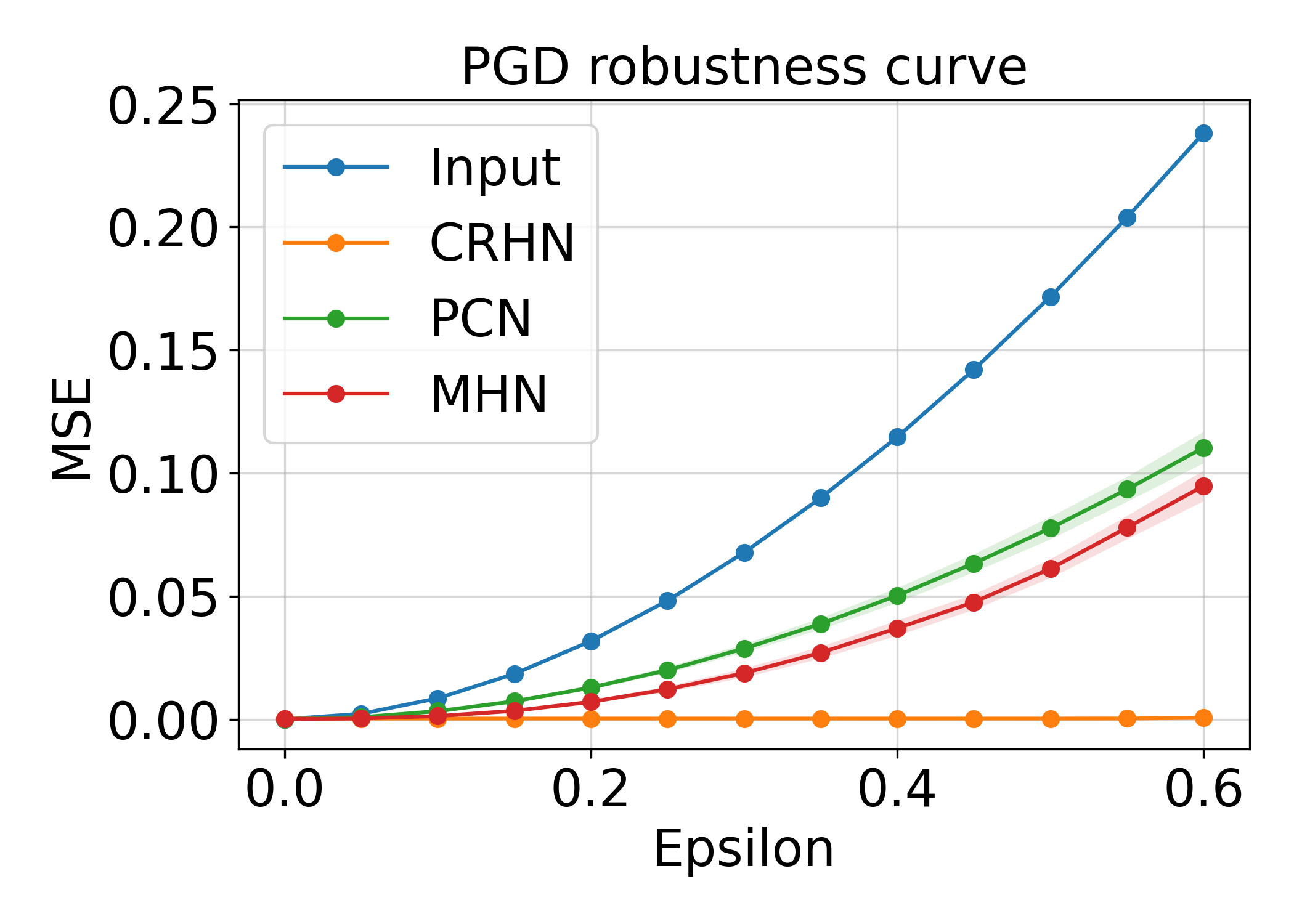}
	\end{subfigure}
	\hfill
	\begin{subfigure}[t]{0.24\textwidth}
		\centering
		\includegraphics[width=\linewidth]{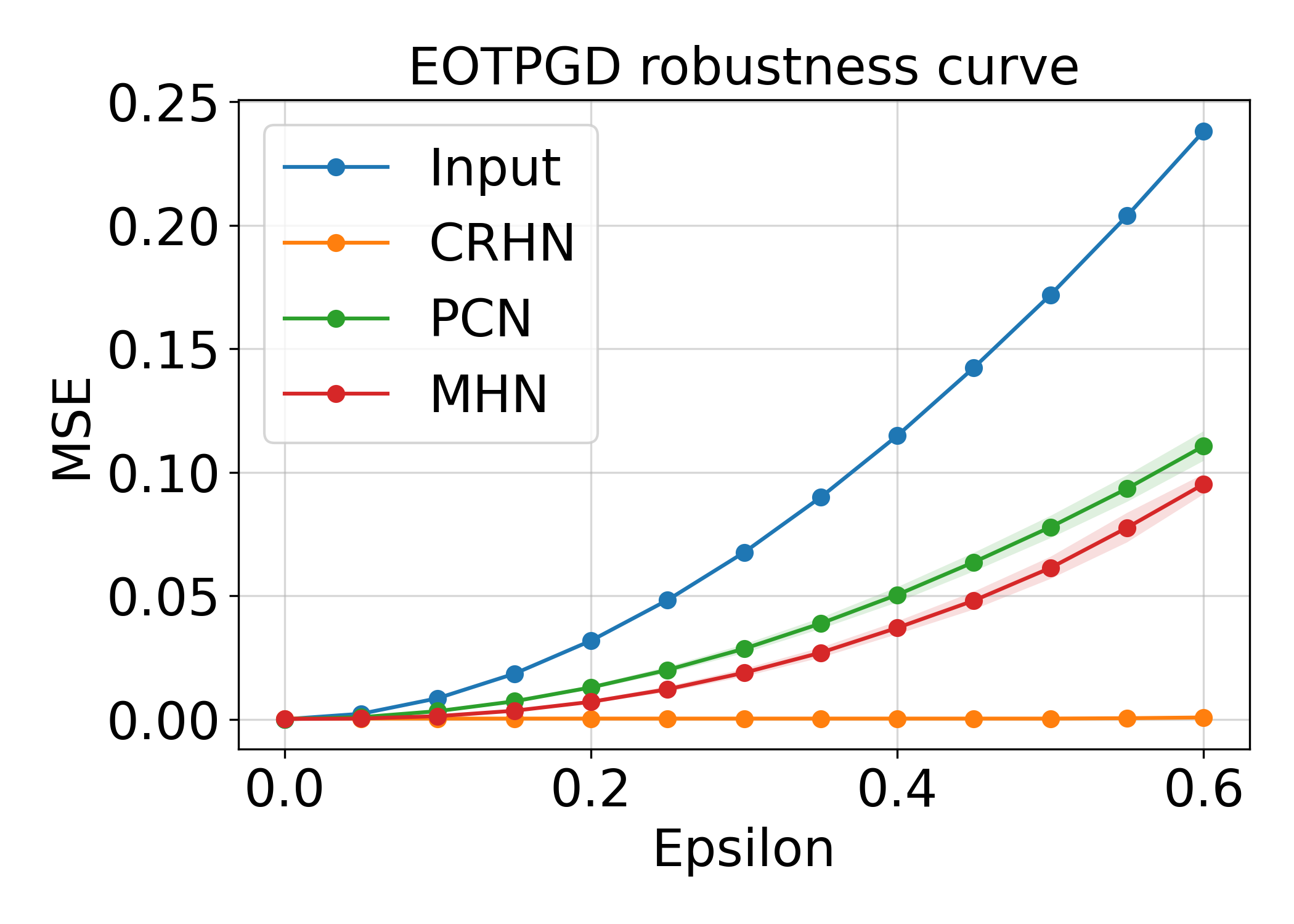}
	\end{subfigure}
	\caption{Quantitative robustness evaluation under increasing adversarial perturbation strength. Each subplot shows the reconstruction error (MSE) as a function of perturbation magnitude $\epsilon$ for a specific attack method, including FGSM, FFGSM, MIFGSM, NIFGSM, BIM, PGD, EOTPGD, and DIFGSM. All models are trained to memorize 250 STL images resized to $96\times96$. The curves correspond to the input perturbation and the reconstruction results produced by MHNs, PCNs, and CRHNs, respectively. As $\epsilon$ increases, CRHNs consistently maintain lower reconstruction error compared to MHNs and PCNs, demonstrating stronger robustness to adversarial perturbations.}
	\label{fig:retrieval_dynamics_eps_250}
\end{figure*}

Table~\ref{tab:attack_eval_50_100} presents a quantitative comparison of the reconstruction error (MSE) of CRHNs, PCNs, and MHNs under various adversarial attack methods for two experimental settings, with memory sizes of 50 and 100. Across most attack methods, MHNs and PCNs exhibit similar performance trends, with PCNs generally achieving slightly lower errors than MHN. The performance gap between these two models is relatively small, suggesting that both are similarly sensitive to adversarial perturbations. In contrast, CRHNs consistently achieves significantly lower reconstruction errors, often by an order of magnitude, indicating substantially stronger robustness. To further validate the observed performance gains, we conduct Welch’s t-tests comparing CRHNs with MHNs and PCNs across all attack methods. The results show that CRHNs achieves significantly lower reconstruction error in all cases ($p < 0.01$), confirming that the improvements are statistically significant rather than due to random variation.

\subsection{Robustness Evaluation under Varying Perturbation Strength}

\begin{table*}[htbp]
	\centering
	\begin{threeparttable}
		\caption{Ablation study comparing full CRHNs and CRHNs without RHNs under Gaussian noise ($\sigma=0.7$) and partial masking (50 columns occluded). Lower MSE indicates better reconstruction performance.}
		\label{tab:ablation_combined}
		\setlength{\tabcolsep}{5pt}
		\begin{tabular}{lcc@{\hspace{1.2em}}cc}
			\toprule
			& \multicolumn{2}{c}{Gaussian Noise} & \multicolumn{2}{c}{Masking} \\
			\cmidrule(lr){2-3} \cmidrule(lr){4-5}
			Memory Size & CRHNs & CRHNs w/o RHNs & CRHNs & CRHNs w/o RHNs \\
			\midrule
			50  & $0.0002 \pm 0.0000$ & $0.0002 \pm 0.0001$ & $0.0002 \pm 0.0000$ & $0.0101 \pm 0.0093$ \\
			100 & $0.0002 \pm 0.0000$ & $0.0004 \pm 0.0002$ & $0.0002 \pm 0.0000$ & $0.0038 \pm 0.0040$ \\
			250 & $0.0244 \pm 0.0136$ & $0.0607 \pm 0.0338$ & $0.0003 \pm 0.0000$ & $0.0006 \pm 0.0002$ \\
			500 & $0.1133 \pm 0.0459$ & $0.3057 \pm 0.0920$ & $0.0184 \pm 0.0352$ & $0.1000 \pm 0.0832$ \\
			\bottomrule
		\end{tabular}
		\begin{tablenotes}
			\footnotesize
			\item[1] Both models share the same convolutional encoder–decoder. The ablated model removes the RHN module, isolating the contribution of attractor-based retrieval dynamics.
		\end{tablenotes}
	\end{threeparttable}
\end{table*}

\begin{figure*}[htbp]
	\centering
	\begin{subfigure}[t]{0.315\textwidth}
		\centering
		\includegraphics[width=\linewidth]{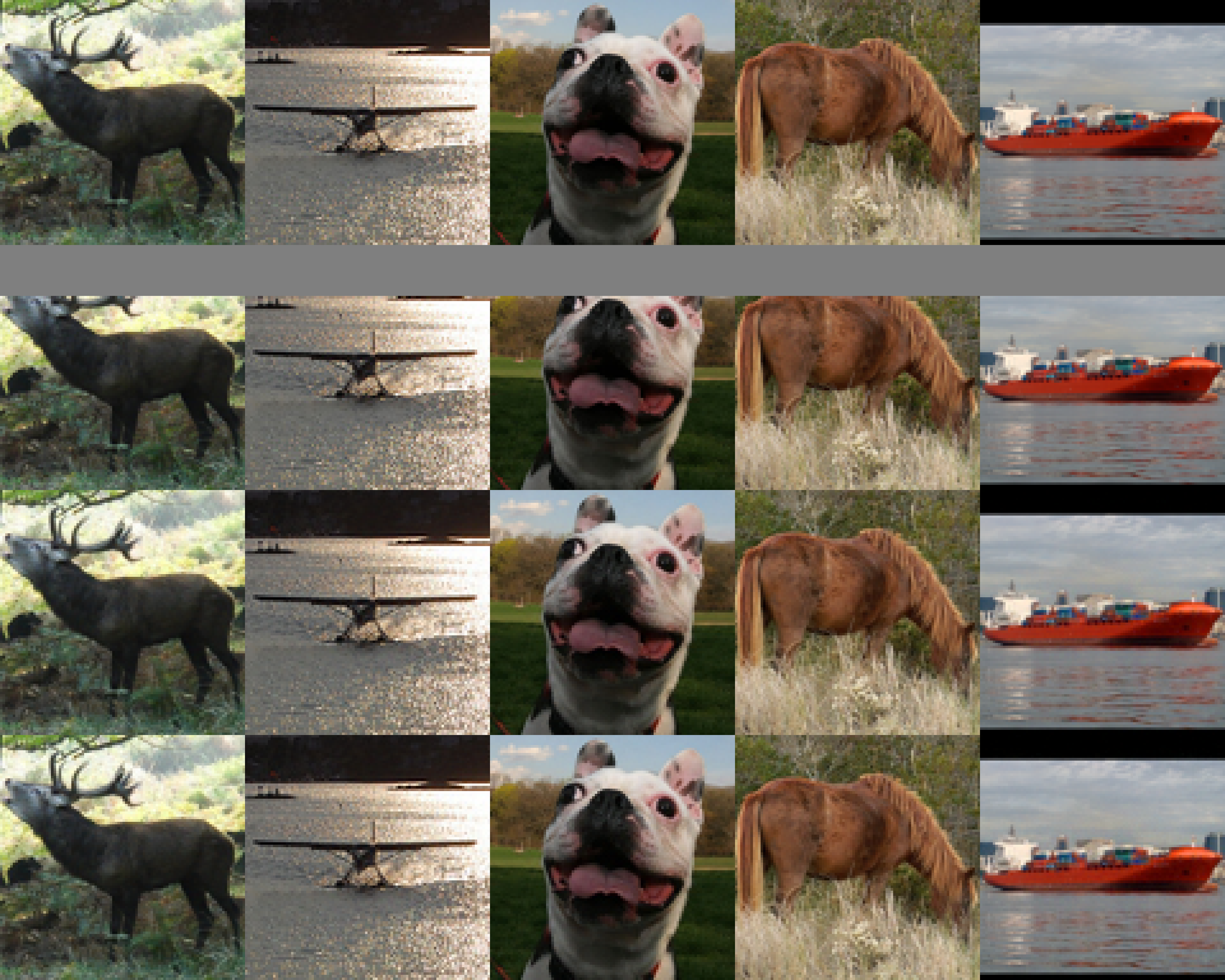}
		\caption{20 Rows are Corrupted}
		\label{fig:retrieval_dynamics_mask_ablation_20}
	\end{subfigure}
	\hfill
	\begin{subfigure}[t]{0.315\textwidth}
		\centering
		\includegraphics[width=\linewidth]{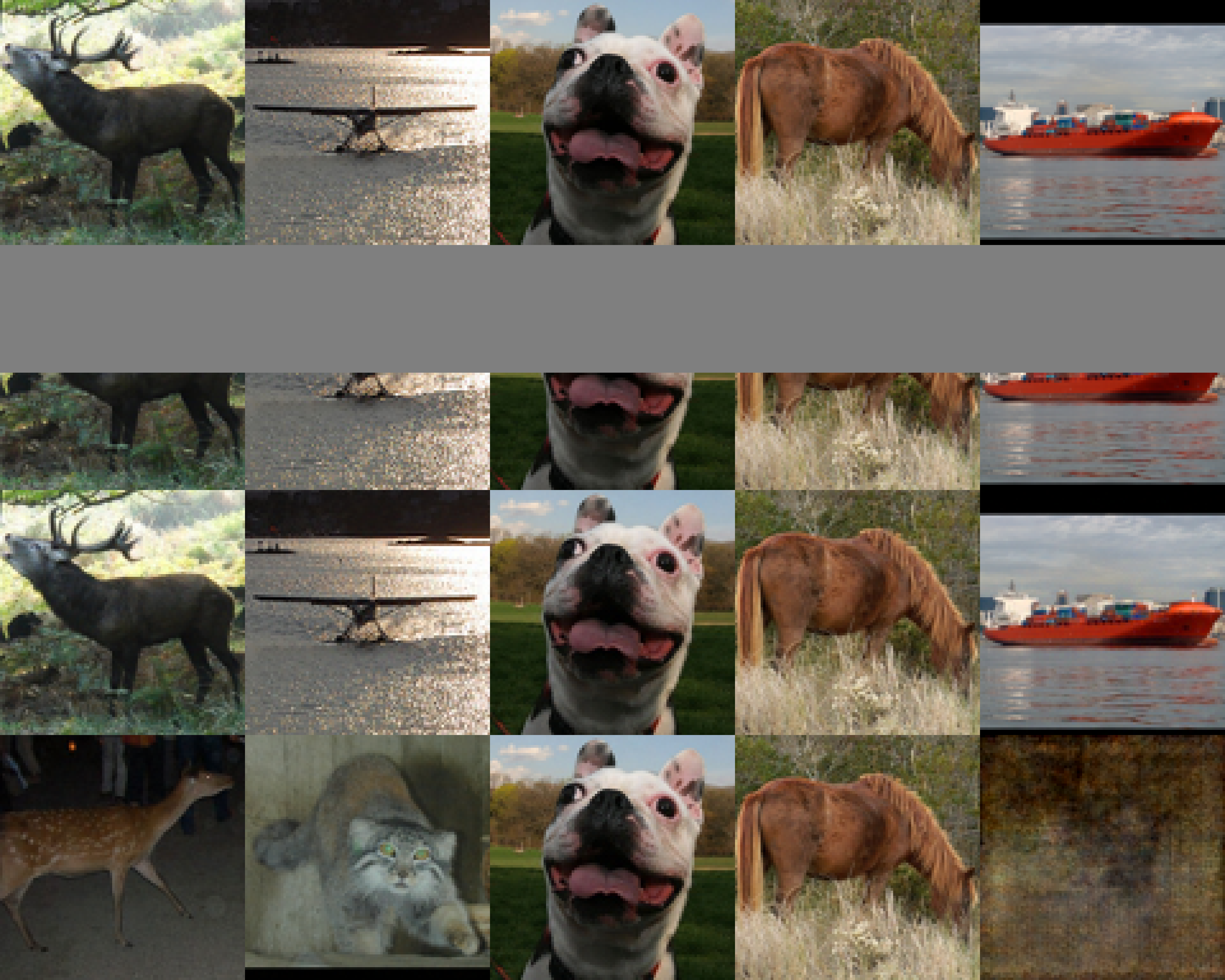}
		\caption{50 Rows are Corrupted}
		\label{fig:retrieval_dynamics_mask_ablation_50}
	\end{subfigure}
	\hfill
	\begin{subfigure}[t]{0.315\textwidth}
		\centering
		\includegraphics[width=\linewidth]{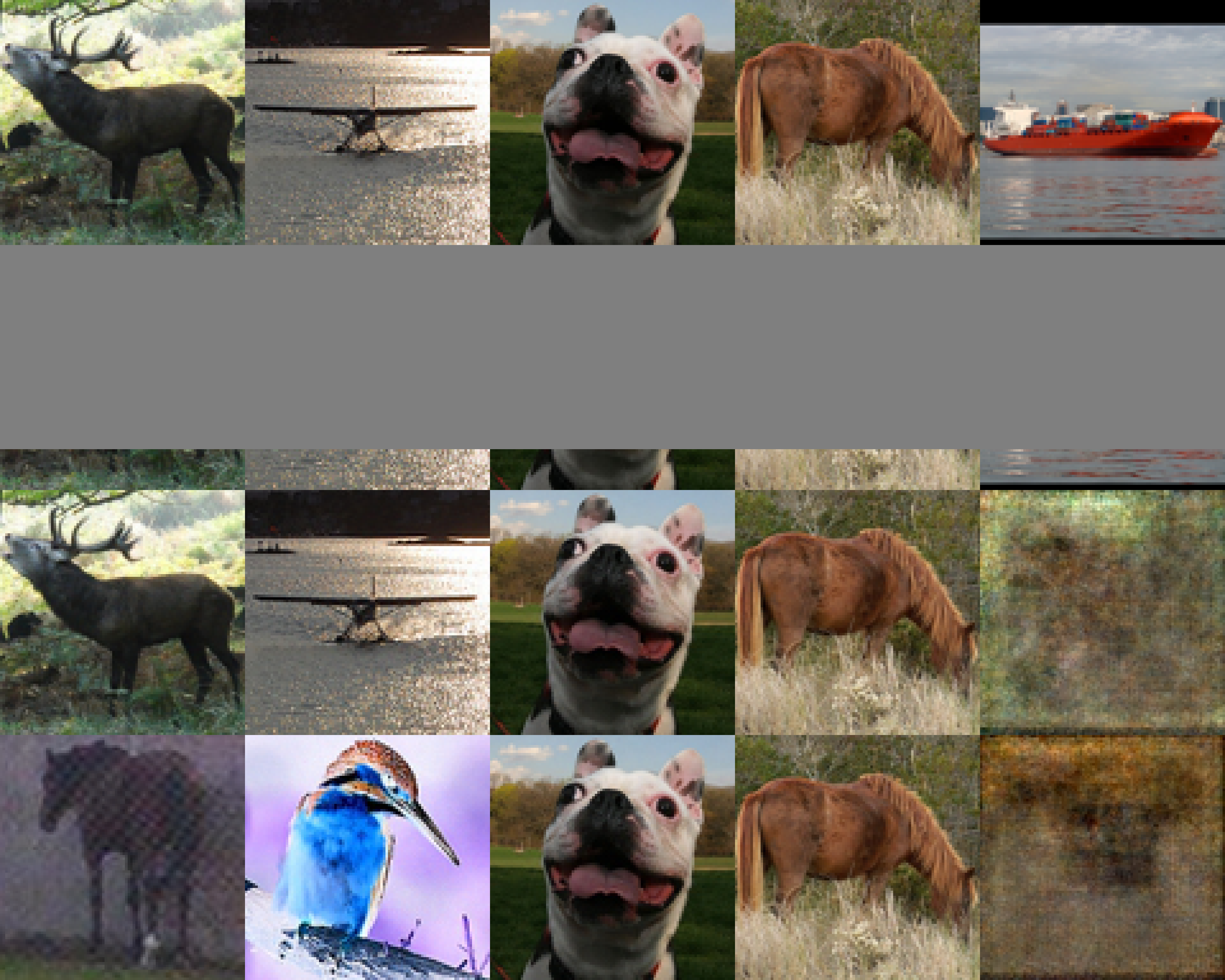}
		\caption{80 Rows are Corrupted}
		\label{fig:retrieval_dynamics_mask_ablation_80}
	\end{subfigure}
	\vspace{0.1em}
	\begin{subfigure}[t]{0.315\textwidth}
		\centering
		\includegraphics[width=\linewidth]{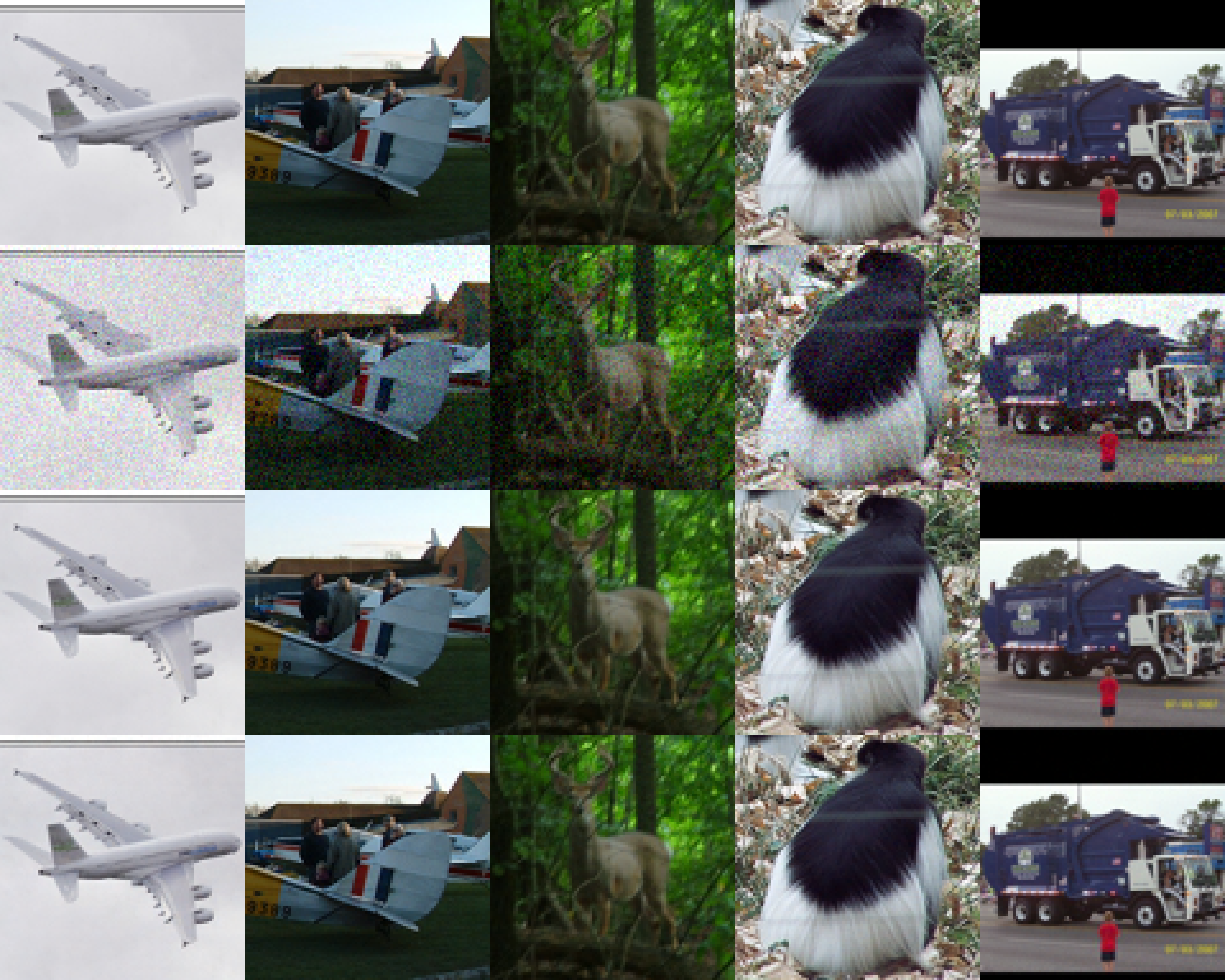}
		\caption{Gaussian Noise $\sigma=0.1$}
		\label{fig:retrieval_dynamics_albation_noise_0.1}
	\end{subfigure}
	\hfill
	\begin{subfigure}[t]{0.315\textwidth}
		\centering
		\includegraphics[width=\linewidth]{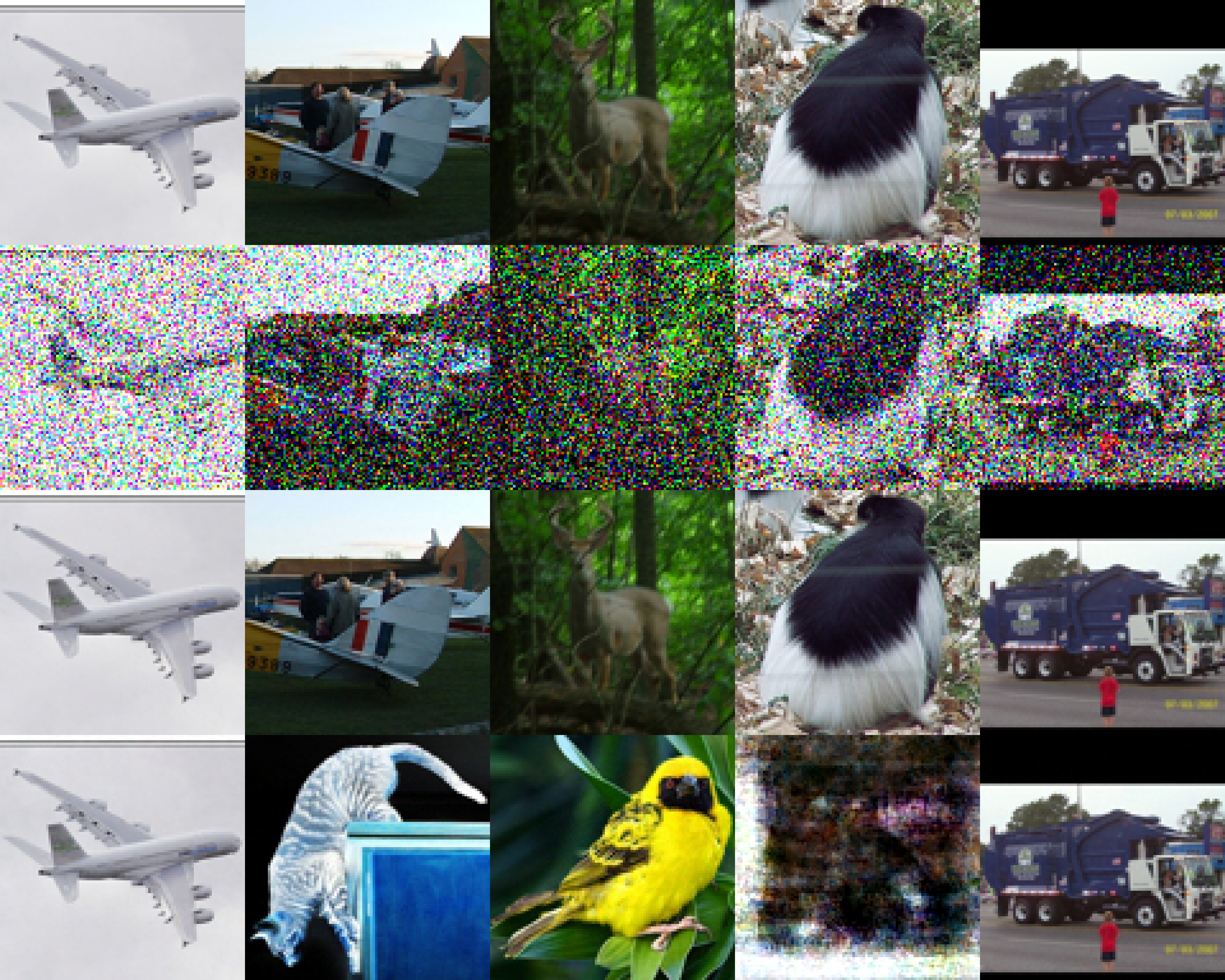}
		\caption{Gaussian Noise $\sigma=0.7$}
		\label{fig:retrieval_dynamics_ablation_noise_0.7}
	\end{subfigure}
	\hfill
	\begin{subfigure}[t]{0.315\textwidth}
		\centering
		\includegraphics[width=\linewidth]{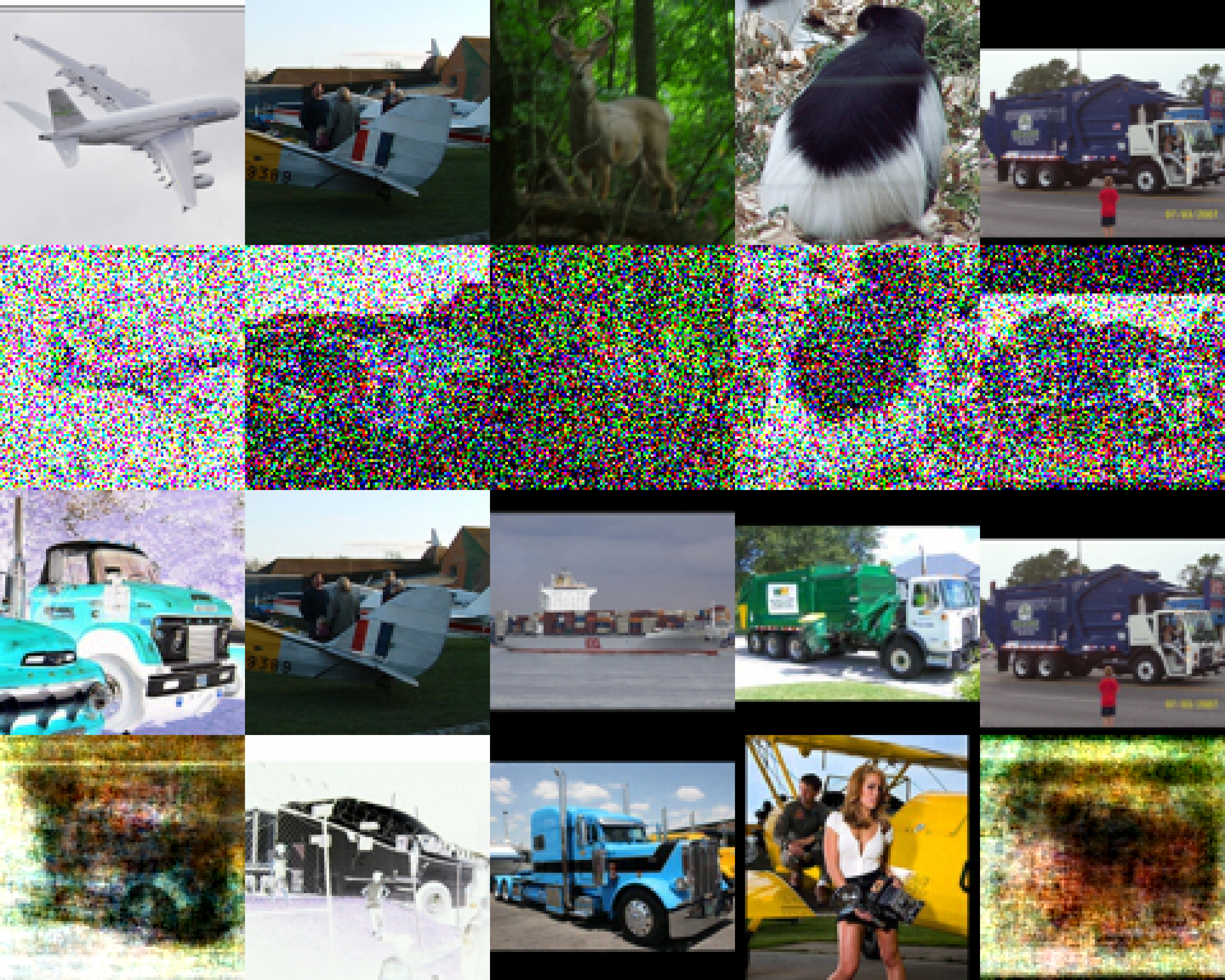}
		\caption{Gaussian Noise $\sigma=1.1$}
		\label{fig:retrieval_dynamics_ablation_noise_1.1}
	\end{subfigure}
	\caption{Qualitative retrieval comparison on the STL dataset, where all models are trained to memorize 500 images resized to $96\times96$. From top to bottom, the rows correspond to the original images, the corrupted and noisy inputs, and the retrieval results obtained by CRHNs, CRHNs without RHNs modules, respectively.}
	\label{fig:retrieval_dynamics_ablation_corrupt_noise}
\end{figure*}

As shown in Figure~\ref{fig:retrieval_dynamics_eps_250}, the robustness of the models is quantitatively evaluated under increasing adversarial perturbation strength across multiple attack methods, including FGSM, FFGSM, MIFGSM, NIFGSM, DIFGSM, BIM, PGD, and EOTPGD. Each subplot reports the reconstruction error (MSE) as a function of the perturbation magnitude $\epsilon$, allowing a direct comparison of how retrieval performance degrades under progressively stronger attacks.

A consistent trend across all attack methods is that the reconstruction error of the perturbed input increases steadily with $\epsilon$, while the reconstruction errors of all three models remain lower than that of the corresponding input. This indicates that all models exhibit a certain degree of noise-filtering capability, as the retrieval process is able to partially suppress adversarial perturbations rather than simply reproducing the corrupted input.

Among the three models, CRHNs demonstrate the strongest robustness. In particular, its reconstruction error remains nearly unchanged over a broad range of perturbation strengths. This behavior suggests that CRHNs are able to preserve stable attractor-based retrieval under moderate adversarial perturbations and only degrades when the perturbation becomes sufficiently strong to push the input outside the effective basin of attraction.

In contrast, MHNs and PCNs exhibit broadly similar sensitivity to adversarial perturbations. For attackers such as FGSM, FFGSM, PGD, and EOTPGD, the performance of MHNs is slightly better than that of PCNs, although the gap between them is relatively small compared to the difference between either model and CRHNs. As $\epsilon$ increases, both models show a more gradual increase in reconstruction error, indicating weaker robustness to adversarial perturbations. For attackers such as DIFGSM, NIFGSM, MIFGSM, and BIM, at smaller values of $\epsilon$, PCNs performs slightly worse than MHNs. However, as $\epsilon$ increases (e.g., to 0.5), MHNs becomes slightly worse than PCNs, although the difference remains small.

In summary, Figure~\ref{fig:retrieval_dynamics_eps_250} shows that all three associative memory models provide some degree of perturbation suppression, but CRHNs are substantially more stable under increasing attack intensity. This result is consistent with the attractor-based retrieval mechanism of CRHNs, which preserves stored patterns more effectively under adversarial perturbations.

\subsection{Ablation Study}

Table~\ref{tab:ablation_combined} presents the ablation results comparing the full CRHNs with a variant without the RHNs module under Gaussian noise and partial masking. Since both models share the same convolutional encoder–decoder and latent interface, this comparison isolates the contribution of the RHNs.

From the results, it can be observed that both models achieve comparable performance when the memory size is small (e.g., 50 and 100 images), particularly under mild perturbations. This suggests that the convolutional encoder–decoder alone is sufficient to reconstruct simple or weakly corrupted patterns.

However, as the memory size increases, a clear performance gap emerges. For memory sizes of 250 and 500, the full CRHNs consistently achieves significantly lower reconstruction error than the ablated model, especially under strong perturbations. For example, under Gaussian noise with memory size 500, the reconstruction error of CRHNs are substantially lower than that of the model without RHNs. A similar trend is observed under masking, where the performance degradation of the ablated model becomes more pronounced.

These results indicate that the RHNs module play a critical role in maintaining robustness when the memory space becomes more complex or when the input is severely corrupted. While the encoder–decoder provides a meaningful latent representation, it lacks the ability to actively correct errors. In contrast, the RHNs introduces iterative attractor dynamics that refine the latent state and guide it toward stored patterns.

The ablation study demonstrates that the robustness of CRHNs cannot be attributed solely to the convolutional representation. Instead, it arises from the combination of structured encoding and attractor-based retrieval, with the RHNs serving as the key mechanism for error correction and stability.

\section{Conclusion and Future Work}\label{sec:conclusion}

In this paper, we proposed the CRHN, a novel auto-associative memory architecture that integrates convolutional representations with structured attractor dynamics. By combining a convolutional encoder–decoder with a recurrent RHN and a discrete latent interface, the proposed model enables efficient pattern completion on high-dimensional data while preserving the stability properties of energy-based systems.

We provided both theoretical and empirical analysis to demonstrate the effectiveness of CRHNs. From a theoretical perspective, we established the stability of the model through a Lyapunov formulation and showed that semi-orthogonal weight structures prevent the amplification of adversarial perturbations. From an experimental perspective, extensive evaluations on the STL dataset show that CRHNs consistently outperform MHNs and PCNs under a wide range of input degradations, including partial occlusion, additive noise, photometric variations, and gradient-based adversarial attacks. In particular, CRHNs exhibit strong robustness under increasing perturbation strength, maintaining stable reconstruction performance until the perturbations exceed the basin of attraction of stored patterns.

These results suggest that the robustness of CRHNs are fundamentally rooted in its attractor-based retrieval mechanism, where stored patterns correspond to stable equilibrium points in the dynamical system. The combination of structured feature extraction and stable energy dynamics enables CRHNs to effectively suppress noise and recover clean patterns from corrupted inputs.

Despite these promising results, several directions remain for future work. First, CRHNs can be extended to larger-scale datasets and more complex visual domains, such as pattern recognition and object detection, to demonstrate their applicability across diverse tasks. Second, integrating RHNs with modern deep architectures, such as Transformers or multimodal models, may provide new opportunities to combine associative memory with large-scale representation learning. Finally, a deeper theoretical investigation of the relationship between attractor geometry, subspace structure, and adversarial robustness could provide further insights into the design of robust neural systems.

\bibliographystyle{unsrt} 
\bibliography{reference} 

@inproceedings{coates2011analysis,
	title={{An analysis of single-layer networks in unsupervised feature learning}},
	author={Coates, Adam and Ng, Andrew and Lee, Honglak},
	booktitle={Proceedings of the fourteenth international conference on artificial intelligence and statistics},
	pages={215--223},
	year={2011},
	organization={JMLR Workshop and Conference Proceedings}
}

@inproceedings{xie2019improving,
	title={{Improving transferability of adversarial examples with input diversity}},
	author={Xie, Cihang and Zhang, Zhishuai and Zhou, Yuyin and Bai, Song and Wang, Jianyu and Ren, Zhou and Yuille, Alan L},
	booktitle={Proceedings of the IEEE/CVF conference on computer vision and pattern recognition},
	pages={2730--2739},
	year={2019}
}

@inproceedings{fischer2012introduction,
	title={{An introduction to restricted Boltzmann machines}},
	author={Fischer, Asja and Igel, Christian},
	booktitle={Iberoamerican congress on pattern recognition},
	pages={14--36},
	year={2012},
	organization={Springer}
}

@article{wahlheim2025memory,
	title={{Memory updating and the structure of event representations}},
	author={Wahlheim, Christopher N and Zacks, Jeffrey M},
	journal={Trends in cognitive sciences},
	volume={29},
	number={4},
	pages={380--392},
	year={2025},
	publisher={Elsevier}
}

@inproceedings{ramsauer2020hopfield,
  title={{Hopfield Networks is All You Need}},
  author={Ramsauer, Hubert and Sch{\"a}fl, Bernhard and Lehner, Johannes and Seidl, Philipp and Widrich, Michael and Gruber, Lukas and Holzleitner, Markus and Adler, Thomas and Kreil, David and Kopp, Michael K and others},
  booktitle={International Conference on Learning Representations}
}

@article{hu2023sparse,
	title={{On sparse modern hopfield model}},
	author={Hu, Jerry Yao-Chieh and Yang, Donglin and Wu, Dennis and Xu, Chenwei and Chen, Bo-Yu and Liu, Han},
	journal={Advances in neural information processing systems},
	volume={36},
	pages={27594--27608},
	year={2023}
}

@article{yoo2022bayespcn,
	title={Bayespcn: A continually learnable predictive coding associative memory},
	author={Yoo, Jinsoo and Wood, Frank},
	journal={Advances in Neural Information Processing Systems},
	volume={35},
	pages={29903--29914},
	year={2022}
}

@inproceedings{lin2024subspace,
	title={{Subspace Rotation Algorithm for Training Restricted Hopfield Network}},
	author={Lin, Ci and Yeap, Tet and Kiringa, Iluju},
	booktitle={2024 IEEE 36th International Conference on Tools with Artificial Intelligence (ICTAI)},
	pages={740--747},
	year={2024},
	organization={IEEE}
}

@article{RETTER2020116685,
	title = {{All-or-none face categorization in the human brain}},
	journal = {NeuroImage},
	volume = {213},
	pages = {116685},
	year = {2020},
	issn = {1053-8119},
	doi = {https://doi.org/10.1016/j.neuroimage.2020.116685},
	url = {https://www.sciencedirect.com/science/article/pii/S1053811920301725},
	author = {Talia L. Retter and Fang Jiang and Michael A. Webster and Bruno Rossion},
}

@inproceedings{tang2022associative,
	title={{Associative Memory Via Covariance-Learning Predictive Coding Networks}},
	author={Tang, Mufeng and Salvatori, Tommaso and Song, Yuhang and Millidge, Beren and Lukasiewicz, Thomas and Bogacz, Rafal},
	booktitle={36th Conference on Neural Information Processing Systems (NeurIPS 2022)},
	year={2022}
}

@article{santos2025hopfield,
	title={{Hopfield-fenchel-young networks: A unified framework for associative memory retrieval}},
	author={Santos, Saul and Niculae, Vlad and McNamee, Daniel and Martins, Andr{\'e} FT},
	journal={Journal of Machine Learning Research},
	volume={26},
	number={265},
	pages={1--51},
	year={2025}
}

@article{sagodi2024back,
	title={{Back to the continuous attractor}},
	author={S{\'a}godi, {\'A}bel and Mart{\'\i}n-S{\'a}nchez, Guillermo and Sok{\'o}{\l}, Piotr and Park, Il Memming},
	journal={Advances in Neural Information Processing Systems},
	volume={37},
	pages={66856--66906},
	year={2024}
}

@inproceedings{salvatori2023feature,
	title={{Associative Memories in the Feature Space}},
	author={Salvatori, Tommaso and Millidge, Beren and Song, Yuhang and Bogacz, Rafal and Lukasiewicz, Thomas},
	booktitle={ECAI 2023: Proceedings of the 26th European Conference on Artificial Intelligence},
	publisher={IOS Press},
	year={2023}
}

@article{rolls2013mechanisms,
	title={{The mechanisms for pattern completion and pattern separation in the hippocampus}},
	author={Rolls, Edmund T},
	journal={Frontiers in systems neuroscience},
	volume={7},
	pages={74},
	year={2013},
	publisher={Frontiers Media SA}
}

@article{rolls1995model,
	title={{A model of the operation of the hippocampus and entorhinal cortex in memory}},
	author={Rolls, Edmund T},
	journal={International Journal of Neural Systems},
	year={1995}
}

@article{marr1971simple,
	author = {Marr, D.},
	title = {{Simple memory: a theory for archicortex}},
	journal = {Philosophical Transactions of the Royal Society of London. B, Biological Sciences},
	volume = {262},
	number = {841},
	pages = {23-81},
	year = {1971},
	month = {07},
	issn = {0080-4622},
	doi = {10.1098/rstb.1971.0078},
	url = {https://doi.org/10.1098/rstb.1971.0078},
	eprint = {https://royalsocietypublishing.org/rstb/article-pdf/262/841/23/336453/rstb.1971.0078.pdf},
}

@article{liu2018adv,
	title={{Adv-bnn: Improved adversarial defense through robust bayesian neural network}},
	author={Liu, Xuanqing and Li, Yao and Wu, Chongruo and Hsieh, Cho-Jui},
	journal={arXiv preprint arXiv:1810.01279},
	year={2018}
}

@inproceedings{dong2018boosting,
	title={{Boosting adversarial attacks with momentum}},
	author={Dong, Yinpeng and Liao, Fangzhou and Pang, Tianyu and Su, Hang and Zhu, Jun and Hu, Xiaolin and Li, Jianguo},
	booktitle={Proceedings of the IEEE conference on computer vision and pattern recognition},
	pages={9185--9193},
	year={2018}
}

@article{lin2019nesterov,
	title={{Nesterov accelerated gradient and scale invariance for adversarial attacks}},
	author={Lin, Jiadong and Song, Chuanbiao and He, Kun and Wang, Liwei and Hopcroft, John E},
	journal={arXiv preprint arXiv:1908.06281},
	year={2019}
}

@article{wong2020fast,
	title={{Fast is better than free: Revisiting adversarial training}},
	author={Wong, Eric and Rice, Leslie and Kolter, J Zico},
	journal={arXiv preprint arXiv:2001.03994},
	year={2020}
}

@inproceedings{lin2023basin,
	title={{On the Basin of Attraction and Capacity of Restricted Hopfield Network as an Auto-Associative Memory}},
	author={Lin, Ci and Yeap, Tet and Kiringa, Iluju},
	booktitle={2023 International Conference on Cyber-Enabled Distributed Computing and Knowledge Discovery (CyberC)},
	pages={146--154},
	year={2023},
	organization={IEEE}
}

@inproceedings{lin2025rhnrobust,
	title={{Restricted Hopfield Networks are Resilient to Adversarial Perturbations}},
	author={Lin, Ci and Yeap, Tet and Kiringa, Iluju and Zhang, Biwei},
	booktitle={2025 IEEE 7th International Conference on Cognitive Machine Intelligence (CogMI)},
	pages={75--85},
	year={2025},
	organization={IEEE}
}

@article{yeap2021implementation,
	title={Implementation of an associative memory using a restricted hopfield network},
	author={Yeap, Tet},
	journal={Global Journal of Research In Engineering},
	year={2021}
}

@article{hopfield1982neural,
  title={Neural networks and physical systems with emergent collective computational abilities},
  author={Hopfield, John J},
  journal={Proceedings of the national academy of sciences},
  volume={79},
  number={8},
  pages={2554--2558},
  year={1982},
  publisher={National Acad Sciences}
}

@article{hopfield1984neurons,
  title={Neurons with graded response have collective computational properties like those of two-state neurons},
  author={Hopfield, John J},
  journal={Proceedings of the national academy of sciences},
  volume={81},
  number={10},
  pages={3088--3092},
  year={1984},
  publisher={National Acad Sciences}
}

@article{mceliece1987capacity,
  title={The capacity of the Hopfield associative memory},
  author={McEliece, ROBERTJ and Posner, Edwardc and Rodemich, EUGENER and Venkatesh, SANTOSHS},
  journal={IEEE transactions on Information Theory},
  volume={33},
  number={4},
  pages={461--482},
  year={1987},
  publisher={IEEE}
}

@mastersthesis{halabian2021enhanced,
  title={An Enhanced Learning for Restricted Hopfield Networks},
  author={Halabian, Faezeh},
  year={2021},
  school={Universit{\'e} d'Ottawa/University of Ottawa}
}

@mastersthesis{li2021global,
  title={Global Optimization Techniques Based on Swarm-intelligent and Gradient-free Algorithms},
  author={Li, Futong},
  year={2021},
  school={Universit{\'e} d'Ottawa/University of Ottawa}
}

@article{mkadry2017towards,
	title={{Towards deep learning models resistant to adversarial attacks}},
	author={M{\k{a}}dry, Aleksander and Makelov, Aleksandar and Schmidt, Ludwig and Tsipras, Dimitris and Vladu, Adrian},
	journal={stat},
	volume={1050},
	pages={9},
	year={2017}
}

@article{storkey1999basins,
  title={The basins of attraction of a new Hopfield learning rule},
  author={Storkey, Amos J and Valabregue, Romain},
  journal={Neural Networks},
  volume={12},
  number={6},
  pages={869--876},
  year={1999},
  publisher={Elsevier}
}

@article{krotov2016dense,
  title={Dense associative memory for pattern recognition},
  author={Krotov, Dmitry and Hopfield, John J},
  journal={Advances in neural information processing systems},
  volume={29},
  year={2016}
}

@article{krotov2018dense,
  title={Dense associative memory is robust to adversarial inputs},
  author={Krotov, Dmitry and Hopfield, John},
  journal={Neural computation},
  volume={30},
  number={12},
  pages={3151--3167},
  year={2018},
  publisher={MIT Press One Rogers Street, Cambridge, MA 02142-1209, USA journals-info~…}
}

@article{krotov2020large,
  title={Large associative memory problem in neurobiology and machine learning},
  author={Krotov, Dmitry and Hopfield, John},
  journal={arXiv preprint arXiv:2008.06996},
  year={2020}
}

@article{ren2021adversarial,
  title={{Adversarial examples: attacks and defenses in the physical world}},
  author={Ren, Huali and Huang, Teng and Yan, Hongyang},
  journal={International Journal of Machine Learning and Cybernetics},
  volume={12},
  number={11},
  pages={3325--3336},
  year={2021},
  publisher={Springer}
}

@inproceedings{newaz2020adversarial,
  title={{Adversarial attacks to machine learning-based smart healthcare systems}},
  author={Newaz, AKM Iqtidar and Haque, Nur Imtiazul and Sikder, Amit Kumar and Rahman, Mohammad Ashiqur and Uluagac, A Selcuk},
  booktitle={GLOBECOM 2020-2020 IEEE Global Communications Conference},
  pages={1--6},
  year={2020},
  organization={IEEE}
}

@inproceedings{apruzzese2019addressing,
  title={{Addressing adversarial attacks against security systems based on machine learning}},
  author={Apruzzese, Giovanni and Colajanni, Michele and Ferretti, Luca and Marchetti, Mirco},
  booktitle={2019 11th international conference on cyber conflict (CyCon)},
  volume={900},
  pages={1--18},
  year={2019},
  organization={IEEE}
}

@article{salvatori2021associative,
  title={{Associative memories via predictive coding}},
  author={Salvatori, Tommaso and Song, Yuhang and Hong, Yujian and Sha, Lei and Frieder, Simon and Xu, Zhenghua and Bogacz, Rafal and Lukasiewicz, Thomas},
  journal={Advances in Neural Information Processing Systems},
  volume={34},
  pages={3874--3886},
  year={2021}
}

@article{goodfellow2014explaining,
  title={{Explaining and harnessing adversarial examples}},
  author={Goodfellow, Ian J and Shlens, Jonathon and Szegedy, Christian},
  journal={arXiv preprint arXiv:1412.6572},
  year={2014}
}

@article{badjie2024adversarial,
  title={{Adversarial attacks and countermeasures on image classification-based deep learning models in autonomous driving systems: A systematic review}},
  author={Badjie, Bakary and Cec{\'\i}lio, Jos{\'e} and Casimiro, Antonio},
  journal={ACM Computing Surveys},
  volume={57},
  number={1},
  pages={1--52},
  year={2024},
  publisher={ACM New York, NY}
}

@incollection{kurakin2018adversarial,
  title={{Adversarial examples in the physical world}},
  author={Kurakin, Alexey and Goodfellow, Ian J and Bengio, Samy},
  booktitle={Artificial intelligence safety and security},
  pages={99--112},
  year={2018},
  publisher={Chapman and Hall/CRC}
}

@article{spratling2017review,
  title={{A review of predictive coding algorithms}},
  author={Spratling, Michael W},
  journal={Brain and cognition},
  volume={112},
  pages={92--97},
  year={2017},
  publisher={Elsevier}
}

@article{schonemann1966generalized,
  title={{A generalized solution of the orthogonal procrustes problem}},
  author={Sch{\"o}nemann, Peter H},
  journal={Psychometrika},
  volume={31},
  number={1},
  pages={1--10},
  year={1966},
  publisher={Springer}
}

@article{bao2022capacity,
  title={{The capacity of the dense associative memory networks}},
  author={Bao, Han and Zhang, Richong and Mao, Yongyi},
  journal={Neurocomputing},
  volume={469},
  pages={198--208},
  year={2022},
  publisher={Elsevier}
}

@article{munakata2004hebbian,
  title={{Hebbian learning and development}},
  author={Munakata, Yuko and Pfaffly, Jason},
  journal={Developmental science},
  volume={7},
  number={2},
  pages={141--148},
  year={2004},
  publisher={Wiley Online Library}
}

@article{gardner1988space,
  title={{The space of interactions in neural network models}},
  author={Gardner, Elizabeth},
  journal={Journal of physics A: Mathematical and general},
  volume={21},
  number={1},
  pages={257},
  year={1988},
  publisher={IOP Publishing}
}

@article{gardner1988optimal,
  title={{Optimal storage properties of neural network models}},
  author={Gardner, Elizabeth and Derrida, Bernard},
  journal={Journal of Physics A: Mathematical and general},
  volume={21},
  number={1},
  pages={271},
  year={1988},
  publisher={IOP Publishing}
}

@article{tang2023recurrent,
  title={{Recurrent predictive coding models for associative memory employing covariance learning}},
  author={Tang, Mufeng and Salvatori, Tommaso and Millidge, Beren and Song, Yuhang and Lukasiewicz, Thomas and Bogacz, Rafal},
  journal={PLoS computational biology},
  volume={19},
  number={4},
  pages={e1010719},
  year={2023},
  publisher={Public Library of Science San Francisco, CA USA}
}

@article{tscshantz2023hybrid,
  title={{Hybrid predictive coding: Inferring, fast and slow}},
  author={Tscshantz, Alexander and Millidge, Beren and Seth, Anil K and Buckley, Christopher L},
  journal={PLoS Computational Biology},
  volume={19},
  number={8},
  pages={e1011280},
  year={2023},
  publisher={Public Library of Science San Francisco, CA USA}
}

\end{document}